\documentclass[pmlr,twocolumn,10pt]{jmlr} 

\usepackage{booktabs}
\usepackage{siunitx}
\usepackage{amsmath}
\usepackage{amssymb}
\usepackage{nicefrac}       
\usepackage{microtype}      
\usepackage{xcolor}         
\usepackage{graphicx}       
\usepackage{bbm}
\usepackage{paralist}
\usepackage{bbm}
\usepackage{fancyhdr}
\usepackage{graphicx}
\usepackage{stmaryrd}
\usepackage[parfill]{parskip}
\usepackage{mathtools}
\usepackage{cases}
\usepackage{comment}
\usepackage{breqn}
\usepackage{color}
\usepackage{algorithm, algorithmic}
\usepackage{appendix}
\usepackage{xspace}
\usepackage{enumitem}
\usepackage{gensymb}
\usepackage[english]{babel}
\usepackage{todonotes}
\usepackage{wrapfig}
\usepackage{letltxmacro}

\mathchardef\mhyphen="2D

\newcommand{\vertiii}[1]{{\left\vert\kern-0.25ex\left\vert\kern-0.25ex\left\vert #1
    \right\vert\kern-0.25ex\right\vert\kern-0.25ex\right\vert}}

\newcommand{\vect}[1]{{\boldsymbol{#1}}}

\def\balpha{\vect{\alpha}}

\def\bg{{\mathbf{g}}}

\def\bu{{\mathbf{u}}}

\def\bx{{\mathbf{x}}}

\def\bF{{\mathbf{F}}}

\def\bU{{\mathbf{U}}}

\def\bX{{\mathbf{X}}}

\def\cA{\mathcal{A}}

\def\cG{\mathcal{G}}
\def\cH{\mathcal{H}}

\def\cM{\mathcal{M}}

\def\cO{\mathcal{O}}

\def\cR{\mathcal{R}}
\def\cS{\mathcal{S}}

\LetLtxMacro\oldttfamily\ttfamily
\DeclareRobustCommand{\ttfamily}{\oldttfamily\csname ttsize\endcsname}
\newcommand{\setttsize}[1]{\def\ttsize{#1}}%
\setttsize{\footnotesize}%
\def\ttS{\texttt{S}}
\def\ttT{\texttt{T}}

\newcommand{\ope}{WIS}

\DeclareMathOperator*{\argmax}{arg\,max}
\DeclareMathOperator*{\argmin}{arg\,min}
\DeclarePairedDelimiterX{\infdivx}[2]{(}{)}{%
  #1\;\delimsize\|\;#2%
}
\newcommand{\infdiv}{\text{KL}\infdivx}

\newcommand{\doo}{\mathrm{do}}

\newcommand{\model}{\mathcal{M}}

\definecolor{alggreen}{RGB}{36,135,74}
\definecolor{algblue}{RGB}{29,138,168}
\definecolor{algorange}{RGB}{125,75,0}

\theorembodyfont{\upshape}
\theoremheaderfont{\scshape}
\theorempostheader{:}
\theoremsep{\newline}

\jmlrvolume{LEAVE UNSET}
\jmlryear{2022}
\jmlrsubmitted{LEAVE UNSET}
\jmlrpublished{LEAVE UNSET}
\jmlrworkshop{Conference on Health, Inference, and Learning (CHIL) 2022} 

\title[Counterfactually Guided Policy Transfer in Clinical Settings]{Counterfactually Guided Policy Transfer in Clinical Settings}

\author{%
\Name{Taylor W. Killian} \Email{twkillian@cs.toronto.edu}\\
\addr University of Toronto, Vector Institue, Canada
\AND
\Name{Marzyeh Ghassemi} \Email{mghassem@mit.edu}\\
\addr Massachusetts Institute of Technology, United States of America
\AND
\Name{Shalmali Joshi} \Email{shalmali@seas.harvard.edu}\\
\addr CRCS Harvard University (SEAS), United States of America
}

\begin{document}

\maketitle

\begin{abstract}
Domain shift, encountered when using a trained model for a new patient population, creates significant challenges for sequential decision making in healthcare since the target domain may be both data-scarce and confounded.
In this paper, we propose a method for off-policy transfer by modeling the underlying generative process with a causal mechanism. We use informative priors from the source domain to augment counterfactual trajectories in the target in a principled manner. We demonstrate how this addresses data-scarcity in the presence of unobserved confounding. 
The causal parametrization of our sampling procedure guarantees that counterfactual quantities can be estimated from scarce observational target data, maintaining intuitive stability properties.
Policy learning in the target domain is further regularized via the source policy through KL-divergence. 
Through evaluation on a simulated sepsis treatment task, our counterfactual policy transfer procedure significantly improves the performance of a learned treatment policy when assumptions of ``no-unobserved confounding" are relaxed.
\end{abstract}

\paragraph*{Data and Code Availability}
We use data derived from a Sepsis simulator\footnote{\url{https://github.com/clinicalml/gumbel-max-scm}} to demonstrate challenges that partial observability presents when learning treatment policies~\citep{oberst2019counterfactual}. This simulator approximates patient physiology (discretized measurements of heart rate, blood pressure, oxygen concentration, and glucose levels) in response to medical interventions and whether the patient is diabetic. Possible treatments include antibiotics, vasopressors, and mechanical ventilation. Our code, used to augment the simulator and develop the approach described in this paper can be found at~\url{https://github.com/MLforHealth/counterfactual_transfer}.

\section{Introduction}
\label{sec:intro}

As the development of machine learning algorithms matures there is increasing interest in deploying models to complex clinical domains~\citep{ghassemi2018opportunities}. These efforts include the application of reinforcement learning (RL) to sequential decision making and treatment recommendation~\citep{yu2021reinforcement}. However, domain shift between training (source) and deployment (target) patient populations~\citep{finlayson2021clinician} presents challenges largely unaddressed in recent RL work. 
In particular, we are concerned with shifting incidence proportions of (possibly unknown and confounded) comorbidities between independent clinical environments~\citep{subbaswamy2018counterfactual, subbaswamy2020development}. 
These challenges are amplified when few samples are available in the target domain since---for ethical and safety purposes---exploratory new data cannot be collected.
Naively learned treatment policies may overfit to data-collection artefacts~\citep{agniel2018biases}, fail to learn meaningful interventions~\citep{franccois2019overfitting}, or mistime appropriate interventions~\citep{bai2014early}. 
To provide reliable decision support and avoid such errors, principled methods are needed when transferring learned treatment policies between clinical environments.

In this paper we frame transfer in the context of offline, off-policy RL between a data-rich source domain to a data-scarce target domain, as we seek to learn robust policies from fixed observational data. 
We consider two main components of transfer: i) improving estimates of statistical quantities in the target domain, i.e. transition dynamics, and ii) adapting the policy learned within the source domain. We demonstrate that transfer in this setting can be naturally framed as a causal inference problem to answer the question, \emph{``How well can a previously trained policy perform in a new target domain with limited observational data?''}

We consider the effects of data-scarcity and confounding when improving the statistical estimation of physiological responses to treatments (otherwise known as the transition dynamics) of a target patient population. Sub-populations within the observed patient cohort (perhaps categorized by disease phenotype) may exhibit dissimilar behavior in response to treatment, the composition of which may differ in target domain. When critical information about sub-populations is unavailable between domains---creating a measure of unobserved confounding and model misspecification---the accuracy of estimated dynamics will be further constrained. 
To address this, we propose a stochastic regularization of the estimated transition dynamics in the target domain using the estimates derived from the source domain, motivated from principles of counterfactual estimation~\citep{pearl2009causality}. 
We use this \emph{counterfactual regularization} to provide a form of guided exploration in the target domain as a way to improve the estimated transition statistics.

The second component of transfer is an intelligent use of the source policy, $\pi^{(\texttt{S})}$. Even with extensive exploration, a learned policy in the target domain may fail to converge or learn safe interventions due to regions of low data support~\citep{gottesman2019guidelines} or an inaccurate dynamics model~\citep{sutton2018reinforcement}. To address this, we guide the development of the target policy $\pi^{(\texttt{T})}$ through regularization with $\pi^{(\texttt{S})}$. Trained with more data, $\pi^{(\texttt{S})}$ has been exposed to a more accurate estimate of the dynamics as well as observations not present in the target domain and serves to stabilize $\pi^{(\texttt{T})}$. By \emph{regularizing policy learning}, we avoid undue overconfidence when determining correct treatment decisions in the target. 

We propose a novel approach for policy transfer via a dual-regularization approach in offline settings. Specifically:
\begin{enumerate}
    \item We leverage complementary elements from the source domain to support guided counterfactual sampling in the target domain which facilitates better policy learning with limited data.
    \item We prove that our transfer method, Counterfactually Guided Policy Transfer (CFPT), maintains important stability properties.
    \item We demonstrate, with a simulated clinical task, that CFPT obtains notable performance gains (up to 3x improvement) across domain-shifted and confounded environments. 
\end{enumerate}

\section{Related Work}\label{sec:related_work}

{\bf RL in Health}  The use of RL has been explored in healthcare to develop optimal treatment strategies~\citep{yu2021reinforcement}, despite challenges presented by likely confounded data~\citep{gottesman2019guidelines}. 
RL has been used to address schizophrenia~\citep{shortreed2011informing}, HIV~\citep{ernst2006clinical}, sepsis~\citep{komorowski2018artificial,raghu2018model,fatemi2021medical} and mechanical ventilation~\citep{prasad2017reinforcement}.
There has also been efforts to develop reliable evaluation of learned policies since they cannot be directly tested~\citep{kallus2018balanced,gottesman2019combining, futoma2020popcorn} and often fail to generalize beyond their training data~\citep{futoma2020myth}.

{\bf Transfer learning in RL} 
Transfer learning in RL can improve policy learning in independent target domains~\citep{taylor2009transfer}.
In healthcare settings, transfer learning may enable personalized treatment strategies~\citep{marivate2014quantifying,killian2017robust} and better generalization across clinical environments.  
However, challenges arise as domain shift may induce additional confounding. 
When observations are scarce, transition estimates are prone to error~\citep{mannor2004bias,fard2008variance} limiting the effectiveness of counterfactual inference (the investigation of plausible alternatives to observed data). 
To address this, we propose a novel way to incorporate inductive bias using the source domain's transition statistics indirectly---through counterfactual inference---to leverage sub-spaces of observations that may not be in the target domain.

{\bf Causal Inference in ML} Causal inference has been used to formalize counterfactual investigations of underlying data distributions~\citep{pearl2009causality} and has recently grown to be a major focus within offline RL~\citep{bannon2020causality}.
These foundational concepts provide benefits when addressing domain shift in supervised learning~\citep{rojas2018invariant,arjovsky2019invariant}, decision making~\citep{makar2020estimation,johansson2020generalization} and for policy reuse across multiple environments in simple bandit~\citep{bareinboim2014transportability,lee2018structural,lee2020generalized} and multi-agent settings~\citep{foerster2018counterfactual}. 
Yet, these methods require online data collection, not possible in clinical settings. 

Causal concepts have also been useful evaluating policies learned from observational data~\citep{athey2015machine,raghu2018behaviour} (including partially observed domains~\citep{tennenholtz2020off}). 
Counterfactual reasoning in RL has been used to infer individualized treatment policies in healthcare with hidden confounding as a proxy for missing data~\citep{parbhoo2018cause,parbhoo2020transfer} or long-term effects of treatment selection~\citep{schulam2017reliable}. Yet, each of these approaches rely on large and diverse training data.
Our proposed transfer framework specifically relies on inducing bias~\citep{hessel2019inductive} indirectly by leveraging causal frameworks to incorporate an informative prior from the source domain in a partially observed sequential decision making setup.

{\bf Offline RL} 
When learning from batch data, value function estimates to guide policy development are prone to overestimation~\citep{hasselt2010double} and high variance~\citep{romoff2018reward}. Various efforts regularize the policy learning process to maintain stability and limit extrapolation to states and actions not in the dataset~\citep{fujimoto2019off,kumar2019stabilizing}. Recent offline RL algorithms additionally regularize the learned policy to remain close to observed behavior~\citep{wu2019behavior,wang2020critic} through a KL-divergence penalty. 
We use a similar mechanism to constrain the target policy during learning via a form of regularized policy iteration~\citep{farahmand2016regularized}. To the best of our knowledge, our work is the first to leverage regularized policy iteration for transfer in an offline RL setting.

\section{Preliminaries}
{\bf Causal modeling in RL}\label{sec:scm_pomdp}
Clinical decision making is inherently a sequential process.
We model sequential decision making in this setting as a partially observed Markov decision process (POMDP) formalized by a Structural Causal Model (SCM)~\citep{buesing2018woulda}. An SCM $\cM$ describes the causal mechanisms of a system's observed variables $\bX$ by defining functions $\bF$ that govern the mechanisms, and accounting for independent stochasticity through exogenous, or external, noise variables $\mathbf{U}$. 
In the assumed causal graph, the nodes that  directly influence a variable $X_i$ are called the parents of $X_i$, $\textbf{PA}_i$. The structural equations $f \in \bF$ of $\cM$ define this relationship where {\small $X_i = f(\textbf{PA}_i, U_i)$}. Additional background is provided in the Appendix, Sec.~\ref{app:scm}.

 \begin{figure}[htbp!]
    \centering
    \includegraphics[width=0.85\linewidth]{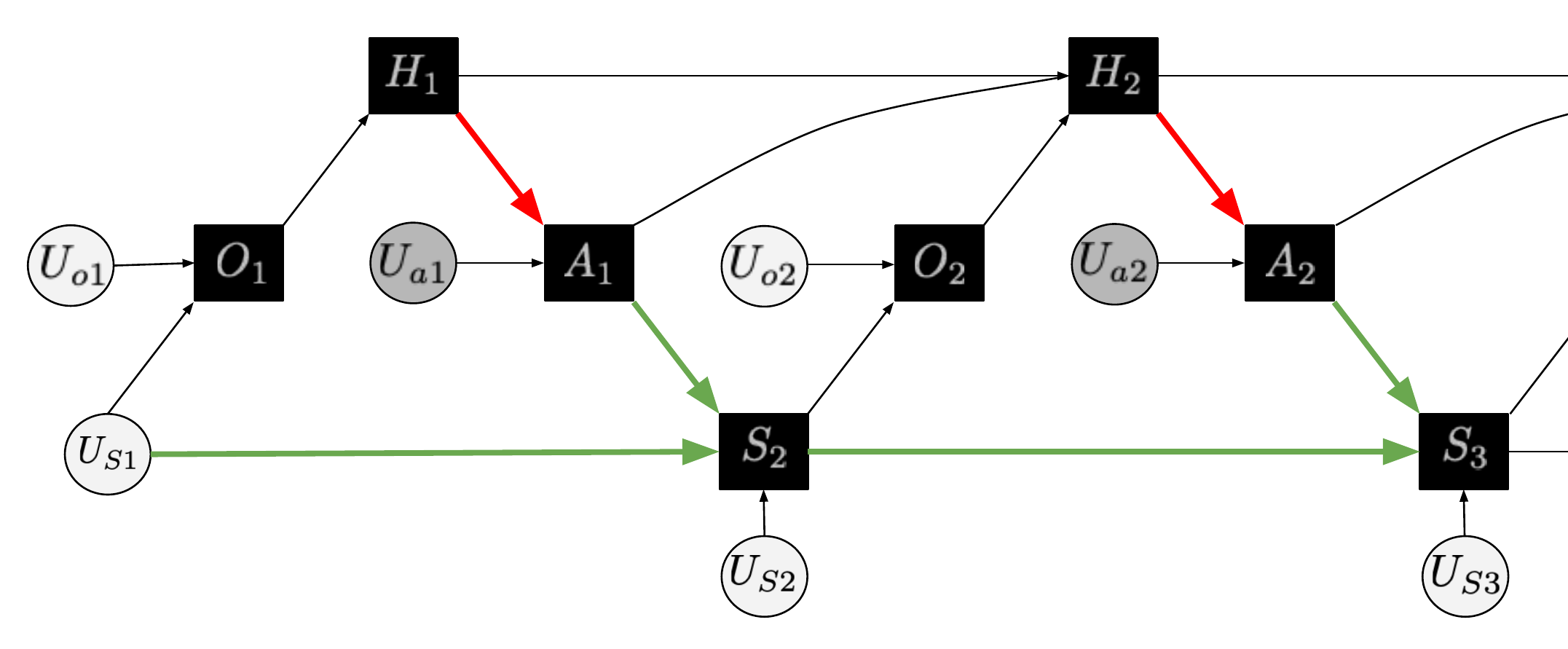}
    \caption{\small SCM of a POMDP from~\citet{buesing2018woulda}. White nodes denote unobserved variables, gray nodes denote observed latent variables and the black nodes are calculated quantities. 
    We assume this structure for both the source and target domain.}
    \label{fig:pomdp_scm}
\end{figure}

{\bf Notation} To facilitate counterfactual inference for transfer from a source domain $\ttS$ to a target domain $\ttT$, we consider finite-state, finite-action episodic POMDPs.
States are denoted as {\small $S_t \in \cS$}, observations by {\small $O_t\in \cO$}, and actions as {\small $A_t \in \cA$} with reward as {\small $R_t = \mathcal{R}(S_t,A_t)$} for {\small $t = \{0,1, \ldots, T\}$}. A POMDP can be represented as an SCM by expressing conditional distributions, e.g. state-transitions {\small $P(S_{t+1} | S_{t}, A_t)$}, as structural equations {\small $S_{t+1}=f(S_t, A_t, U_{St})$}~\citep{pearl2009causality}, shown in Figure~\ref{fig:pomdp_scm} (green edges). 
The relationship between $A_t$ and the observed history {\small $\cH_t=\{O_1, A_1, O_2, \ldots, A_{t-1}, O_t\}$} is governed by the behavior policy $\mu$ (i.e. {\small $\mu (A_t | H_t)$}, red edges) from which trajectories {\small $\tau~=~(S_1, A_1, O_1, \ldots, S_{t-1}, A_{t-1}, O_t)$} are collected with density {\small $p^{\mu}(\tau)$}.

If we choose to execute a learned policy $\pi$ after having observed the behavior policy $\mu$, the functional mapping of the red edges in Figure~\ref{fig:pomdp_scm} changes from $\mu$ to $\pi$. This soft-intervention is denoted by {\small $I'(\mu \rightarrow \pi)$} with the resulting SCM {\small $\cM^{do(I'(\mu \rightarrow \pi))}$}. This induces a modified probability distribution, $P^{do(I'(\mu \rightarrow \pi)) | {I(\mu)}}$ on the POMDP. We denote the corresponding ``counterfactual" random variables with subscripts $X_{I}$ for $do(\mu)$ and $X_{I'}$ for $do(\pi)$. The following procedure outlines how to estimate such a counterfactual distribution:

{\small
\begin{itemize}[leftmargin=*]
    \item[i)] Abduction: estimate posteriors over exogenous noise variables  {\small $p(\bU | \bX)$},~
    \item[ii)] Intervention: execute $\pi$ or $do(I(\mu \rightarrow \pi))$ as if the distribution of the exogenous variables is now fixed to the posterior estimates from step i). 
    \item[iii)] Estimation: estimate the joint distribution of the data, will correspond with $P^{do(I'(\mu \rightarrow \pi)) | {I(\mu)}}$.
\end{itemize} }

In our case, we do not need to characterize the complete distribution of the effect a policy has on the observed data. An estimate of the reward from executing $\pi$ instead of $\mu$ is sufficient. This expected reward under this counterfactual distribution can be estimated from trajectories sampled from $P^{do(I'(\mu \rightarrow \pi)) | {I(\mu)}}$. We denote this reward by $\mathbb{E} [\cR(\tau) | do(\pi)]$. 
This expected reward can also be used to evaluate the average treatment effect under these soft-interventions to determine the value of $\pi$ (see Sec.~\ref{sec:app_ace_connections} in the Appendix). 

In general, it is not always possible to estimate $P^{do(I'(\mu \rightarrow \pi)) | {I(\mu)}}$ and its corresponding expectations. However, we can estimate these quantities from observed data samples if we appropriately restrict the functional mappings $f$. One such choice of these mappings in the case of discrete or categorical states is the Gumbel-Max SCM.

{\bf Gumbel-Max Topdown Sampling.}\label{sec:topdown_sampling} The Gumbel-Max trick enables sampling from categorical distributions $\mathrm{Cat}(\alpha_1,\ldots,\alpha_K)$, where the category $k$ will be selected with probability $\alpha_k$ among K distinct categories~\citep{hazan2012partition,maddison2014sampling,maddison2016concrete}. This sampling procedure rests on inferring Gumbel variables $g_k$ that can be transformed into these probabilities $\alpha_k$.

Without any prior on the Gumbel variables $g^{(\ttT)}_k$, corresponding to the discrete patient states observed in a target domain $\ttT$, the location parameters can be obtained according to the empirical transition probabilities $P^{(\ttT)}$. That is, $p(\balpha) = \delta(\log{P^{\ttT}}(\cdot | \cdot))$ where $\delta$ is the dirac-delta distribution. Sampling from this Gumbel given observation $k'$ can be done using the Topdown procedure\footnote{\url{https://cmaddis.github.io/gumbel-machinery}}.

That is, for a fixed and known $\alpha_k$, the Gumbel corresponding to the observed outcome $k'$, i.e. $g^{(\ttT)}_{k'}$ is itself a Gumbel variable with location parameter $Z = \log{\sum_{k=1}^K \alpha_{k}}$. It follows that the maximum value $k'$ and corresponding Gumbels are independent and the rest of the exogenous variables $g^{(\ttT)}_k~\forall k \neq k'$ are truncated by this maximum value corresponding to $k'$. To leverage information from the source domain $\ttS$, we replace the dirac-delta prior by a mixture of the source and target transition statistics (see Sec.~\ref{sec:aug_gumbel}).

{\bf The Gumbel-Max SCM.}\label{sec:gumbel_max_scm}
\citet{oberst2019counterfactual} introduced the Gumbel-Max SCM, which ensures that counterfactual queries preserve observed outcomes (defined as \emph{counterfactual stability}). In a Gumbel-Max SCM all nodes $\bX$ are discrete random variables with causal mechanisms: 
{\small
\begin{equation}\label{eq:gumbel_max_scm}
    X_i \coloneqq \argmax_{j} \log{p(X_i = j | \textbf{PA}_i)} + g_j
\end{equation}
}
given independent Gumbel variables {\small $\bg~=~\{g_1, g_2, \ldots, g_k\}$}. These structural equations effectively embed the Gumbel-Max trick.  
We parametrize the state transition mechanism of the POMDP using the formulation of Eqt.~\ref{eq:gumbel_max_scm}. This means that exogenous variables are restricted to Gumbel variables such that $U_{S} \triangleq \bg$ (for all time-steps).

To ensure counterfactual expectations $\mathbb{E} [\cR(\tau) | do(\pi)]$ are identifiable from observational data from policy $\mu$, we need an additional property called Counterfactual Stability:

{\small
\begin{definition}{Counterfactual Stability}\label{def:stability}
An SCM over discrete random variables is counterfactually stable if:
$$\frac{p_i'}{p_i} \geq \frac{p_j'}{p_j} \ \Rightarrow \ P^{do(I') | X_{I}=i}(X=j) = 0, \quad \forall j \neq i$$ where $p_i = p(X_I=i)$ and $p'_{i} = p(X_{I'}=i)$.
\end{definition}
}

Defining the structural equations in this manner acts as a constraint on the POMDP, enforcing counterfactual stability (by definition of Gumbel-Max SCMs) when considering alternative state transitions. This ensures that inferred patient outcomes change only when the relative likelihoods also change. Our counterfactual regularization further maintains this property when sampling counterfactual trajectories in the target domain after incorporating source transition estimates, outlined in Sec.~\ref{sec:aug_gumbel}. In effect all intermediate quantities remain estimable from offline data, allowing principled offline transfer.

\section{Counterfactually Guided Policy Transfer}
\label{sec:policy_transfer}

In this section we introduce a framework for transferring learned treatment policies in offline settings; meaning we only have access to sequences of observations (trajectories) $\tau$ without the ability to interact with the intended target domain. By modeling the common generative process between domains with a causal mechanism we are able to constrain policy learning in the target to refrain from unsafe behaviors, even when presented with a different and unknown mixture of patient sub-populations. 

We now formalize the transfer setting. First, we assume that patients in the source and target domains have comparable health conditions. The primary shift between domains is in the composition of patient sub-populations. That is, we have different proportions of  patient types (e.g. the proportion diabetic patients) in each. We assume that the population composition is unknown, creating unobserved confounding in the underlying causal system.

We assume that the data has been collected previously in the source domain $\ttS$ with some unknown behavioral policy $\mu$ and that an optimal treatment policy {\small $\pi^{(\ttS)}$} has been learned. Further, we assume that the empirical transition matrix, {\small ${P}^{(\ttS)}$}, is accessible. Empirical transition statistics {\small ${P}^{(\ttT)}$} in the target are also available. 
We demonstrate that a learned treatment policy in $\ttT$ can be improved by \begin{inparaenum} \item[1)] appropriately leveraging {\small ${P}^{(\ttS)}$} to improve counterfactual transition estimation in $\ttT$ and \item[2)] regularizing {\small $\pi^{(\ttT)}$} by {\small $\pi^{(\ttS)}$}. \end{inparaenum}

In Sec.~\ref{sec:aug_gumbel} we motivate that data-scarcity and unobserved confounding in $\ttT$ induces model misspecification, requiring careful regularization of the transition statistics. 
We propose a stochastic regularization procedure to alleviate challenges of naively transferring $P^{(\ttS)}$. 
Our main theoretical contribution demonstrates that this procedure maintains counterfactual stability. 
In Sec.~\ref{sec:reg_pi} we outline a second form of regularization to stabilize policy learning in $\ttT$, to avoid overconfidence in regions of little support. These concepts are combined in Sec.~\ref{sec:full_method} to introduce our proposed transfer framework for offline, off-policy RL, Counterfactually Guided Policy Transfer (CFPT).

\subsection{Counterfactual Regularization} 
\label{sec:aug_gumbel}

When membership information of patient sub-populations are known, specific estimates of the transition statistics can be obtained in both domains. However, if the statistical bias in these estimates in $\ttT$ is larger for some sub-population, naive regularization from $\ttS$ can only guarantee improvement for the sub-group with more accurate estimates in $\ttS$ (see Appendix~\ref{app:regintuition}).

To improve estimates of the transition statistics $P^{(\ttT)}$, we need to collect more data from the appropriate counterfactual distribution i.e. $P^{do(I'(\mu \rightarrow \pi)) | {I(\mu)}}$. Since naive regularization of $P^{(\ttT)}$ is insufficient, we leverage exogenous variables in $\ttT$ (the Gumbel variables) related to $P^{(\ttT)}$. According to the SCM formulation, the true posterior over these variables is completely described by the true, yet unknown, transition probabilities. Thus estimates of $P^{(\ttT)}$ can be refined by improving the posterior estimates of the Gumbel variables in the ``Abduction'' step.
The transition statistics $P^{(\ttS)}$ are used to improve these posterior estimates which are then used to infer the Gumbels in $\ttT$. This is done with a stochastic mixture of the estimated statistics from both domains.

Our key insight is that this stochastic regularization is helpful even if the mixture membership information is not known. 
In this case, a composite transition estimate is obtained in both $\ttS$ and $\ttT$ (instead of for each sub-population) which enables a guided sampling procedure in $\ttT$ instead of merely relying on $P^{(\ttT)}$. 

\begin{algorithm}[t!]
    \caption{\small Modified Top-down with informative prior} 
    \label{alg:updated_topdown}
    \begin{algorithmic}[1]
        {\small
        \STATE {\color{darkgray} Repeat each step of a counterfactual rollout, infer $\tau^i$}
        \STATE {\color{darkgray} Note: - $\log{P^{(\texttt{S})}(s'|s,a)} = \log{\alpha}^{(\texttt{S})}$ }
        \STATE {\color{darkgray} \qquad - $\log\hat\alpha^{(\texttt{T})}$ are counterfactual stats via policy $\pi$}
        \STATE{\color{darkgray} \qquad - Sampled observation $k'$}
        \newline
        \STATE \textbf{\textsc{Mixture-Topdown}}(SCM $\model$, $\log{\alpha}^{(\texttt{S})}$, $\log{\alpha}^{(\texttt{T})}$, $\log{\hat{\alpha}}^{(\ttT)}$, mixture param $w^{\ttT}$,  $N'$)
        \STATE \hspace{.05em} {\color{darkgray}// Gather a batch of counterfactual trajectories}
        \FOR{$n'=1,\ldots, N'$}
                \STATE $\rho \sim Bernoulli(w^{(\ttT)})$ 
                \STATE $\log{\alpha} = \rho \log{\alpha^{\ttT}} + (1-\rho) \log{\alpha}^{\ttS}$ 
                \STATE $g_{cf} = \text{Topdown}(\log{\alpha}, 1,k')$
            \STATE $S_{cf}^{n'} = {\argmax}_j \log{\hat{\alpha}^{(\ttT)}} + g_{cf}$ 
            \ENDFOR
            \STATE $\hat{P}^{(\ttT)}$ is the empirical estimate using $\{S_{cf}^{n'}\}_{n'=1}^{N'}$
            }
    \end{algorithmic}
\end{algorithm}

Concretely, we we estimate and employ the posterior ${\small p(\bg^{(\ttS)} | \tau^{(\ttS)})}$ from $\ttS$ as an \emph{informative prior} for the target domain, i.e. {\small$p(\bg^{(\ttT)})=p(\bg^{(\ttS)} | \tau^{(\ttS)})$}. 
This prior is incorporated in a way that maintains \emph{counterfactual stability} in $\ttT$, allowing estimation of the expected rewards in the target domain under any candidate policy from observational data collected locally from a different policy $\mu^{\texttt{S}}$. Normally, sampling discrete state outcomes from transition dynamics when parametrized as Gumbel-max variables leverages a sampling procedure known as Top-down sampling. This sampling procedure is a key component for estimating the expected rewards of a policy in the target domain. We ensure stability by carefully designing a modified Top-down sampling procedure~\citep{maddison2014sampling} when sampling from the posterior over Gumbels {\small $\bg^{(\ttT)}$, i.e., 
\begin{align*}
    p(\bg^{(\ttT)} |\tau^{(\ttT)}, P^{(\ttS)}) &\propto p(\tau^{(\ttT)} | \bg^{(\ttT)})p(\bg^{(\ttT)}) \\
    \ &= p(\tau^{(\ttT)} | \bg^{(\ttT)})p(\bg^{(\ttS)} | P^{(\ttS)})
\end{align*}} given some observed trajectory {\small $\tau^{(\ttT)}$}.
The prior $p(\bg^{(\ttT)})$ corresponding to some state-action pair $s, a$ is given by {\small $p_{s,a}(\bg^{(\ttT)}) = \prod_{i=1}^K f_{\log{P^{(\ttS)}(S'=i|S,A)}}(g_i)$}, where {\small$f_{\log{\alpha}}$} is the density of a Gumbel random variable. 

\begin{figure*}[!htbp]
    \centering
    \includegraphics[width=.9\textwidth]{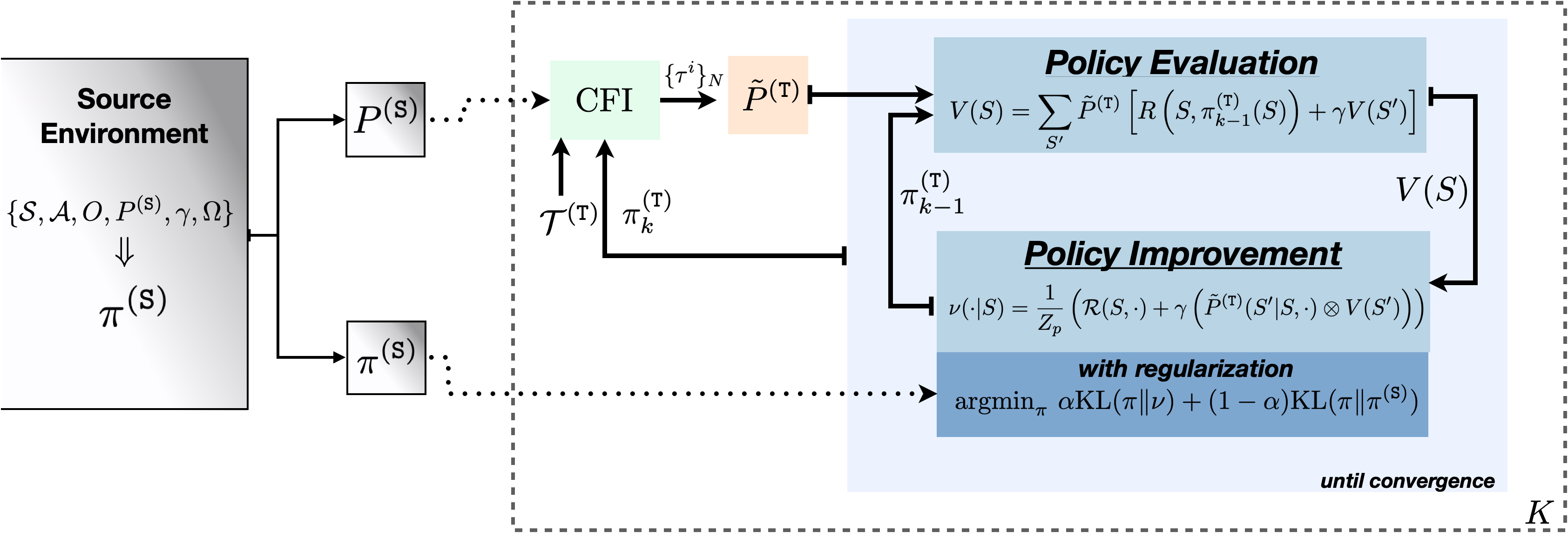}
    \caption{\small Graphical overview of counterfactually guided policy transfer (CFPT), as introduced in this section. 
    Elements from the source domain are used to improve counterfactual inference (CFI) and regularize policy learning within the target domain.}
    \label{fig:procedure_overview}
\end{figure*}

To leverage the prior from $\ttS$ we impose a mixture parametrization over the \emph{posterior} Gumbel distribution conditioned on an observation $k'$ (in $\ttT$):
{\small
\begin{align}
\begin{split}\label{eq:mixture_prior}
p(g^{(\ttT)}_1, \ldots, g^{(\ttT)}_n | k') = &w^{(\ttT)} p(g^{(\ttT)}_1, \ldots, g^{(\ttT)}_n | \log{P^{(\ttT)}}, k')\\ +& w^{(\ttS)} p(g^{(\ttT)}_1, \ldots, g^{(\ttT)}_n | \log{P^{(\ttS)}}, k')
\end{split}
\end{align}
}
where $w^{(\ttS)} =  1-w^{(\ttT)}$. The mixture weight $w$ ($w<1$) is treated as a hyper-parameter determining the amount of regularization provided by $\ttS$. This results in a modified Top-down sampling procedure, summarized in Alg.~\ref{alg:updated_topdown}. Specifically, line $8$ is used to select the Gumbel component from $\ttS$ or $\ttT$ with probability $w^{(\ttT)}$. This component is then provided to sample the Gumbels, given observation $k'$ from $\ttT$ (line $10$). The sample is then used to infer counterfactual states under observation $k'$, ensuring \emph{counterfactual stability} (line $11$). This modified Top-down sampling procedure provides stable counterfactual trajectories in $\ttT$ via regularization from $\ttS$ to form a batch of data to refine a treatment policy $\hat{\pi}^{(\ttT)}$ from. 
The resulting trajectories can be used to re-estimate transition dynamics in the target domain (Alg.~\ref{alg:CF_PI}, line $7$) and can also be thought of as a form of stable exploration in $\ttT$.

\begin{lemma}\label{sec:thm_stability}
The mixture-prior with Modified Top-down sampling preserves counterfactual stability.
\end{lemma}
\begin{proof}
Counterfactual stability is invariant to the choice of prior so long as the gumbel samples are fixed across interventions. Our modified Top-down sampling procedure ensures this. Hence, counterfactual stability is preserved through regularization. The complete proof is in Sec.~\ref{sec:app_stability} in the Appendix.
\end{proof}

\subsection{Regularized Policy Iteration}\label{sec:reg_pi}

The sampling procedure outlined in Section~\ref{sec:aug_gumbel} allows improved estimation of the target domain transition dynamics and evaluation of counterfactual rewards for candidate policies being considered for improvement. Policy iteration (PI) switches between evaluation and improvement steps that estimate then refine a value function $V$ and greedy policy $\pi$. Thus, the evaluation stage of PI can leverage our modified sampling procedure.  Generally, PI may not optimally converge if the MDP is partially observed (e.g. when critical sub-population information is unknown)~\citep{sutton2018reinforcement}. 
When learning a policy $\pi^{(\ttT)}$ in the target domain, the counterfactually sampled batch of trajectories improve the accuracy of the transition matrix used in the evaluation step of PI. 
However, acting greedily with respect to the inferred value function may encourage poor behavior. To guard against overconfident value estimates, we regularize the policy improvement step by $\pi^{(\texttt{S})}$.

We regularize PI (RegPI) in $\ttT$ through minimizing the KL-divergence between the policy distributions over actions, conditioned on the observed state. Due to the discrete and finite causal framework we use to model $P^{(\ttT)}$, the KL regularization is equivalent to log-aggregation~\citep{heskes1998selecting}. This approach is also functionally equivalent to the behavior regularization found in recent offline RL algorithms such as BRAC~\citep{wu2019behavior} and CRR~\citep{wang2020critic}. In this work the policies are not parametrized, so this regularization directly modifies the action distribution rather than constraining gradient updates.

Within the policy improvement step a proposal distribution $\nu(\cdot|s)$ over the actions is generated:
{\small
\begin{equation}
    \nu(\cdot|S) = \frac{1}{Z_p}\left(\mathcal{R}(S,\cdot) + \gamma\left(\tilde P^{(\texttt{T})}(S'|S,\cdot)\otimes\mathbf{V}(S')\right)\right)
\end{equation}
}
where $Z_p$ is a normalization constant and the operator $\otimes$ is used to indicate a Matrix-vector product such that $V(S')$ is combined with $\tilde P^{(\texttt{T})}(S'|S,\cdot)$,  for each action and possible successor state $S'$. We then seek the policy that minimizes the divergence between $\nu(\cdot|S)$ and $\pi^{(\texttt{S})}(\cdot|S)$. That is, 
{\small
\begin{equation}
    \pi^{(\texttt{T})}_{k-1} = \argmin_{\pi} \ \lambda~ \mathrm{KL}(\pi\|\nu) + (1-\lambda)~\mathrm{KL}(\pi\|\pi^{(\texttt{S})}) \label{eqt:kl_minimization}
\end{equation}
} where $\lambda$ is a hyperparameter, selected empirically to determine how much $\pi^{(\texttt{S})}$ influences $\pi^{(\texttt{T})}$. The derivation of Eqt.~\ref{eqt:kl_minimization} and how it is fully implemented are included in the Appendix (see Sec.~\ref{sec:app_kl_aggregation} and Alg.~\ref{alg:CFPT}).

\subsection{Counterfactual Policy Iteration}
\label{sec:full_method}
\begin{algorithm}[ht]
    \caption{Counterfactual Policy Iteration}
    \label{alg:CF_PI}
    \begin{algorithmic}[1]
    {\small
        \STATE \textbf{\textsc{CF-PI}}(SCM $\model$, $\pi^{(\texttt{T})}_0$, $\pi^{(\texttt{S})}$, $P^{(\texttt{S})}$)
            \FOR{$k=1,\ldots, K$}        
                \STATE \hspace{0.5em} {\color{alggreen}// Gather a batch of counterfactual trajectories}
                \hspace{-1em}\STATE $\{h^i\}_{i=1}^N\sim\mathcal H^{(\texttt{T})} \subset \mathcal T$ 
                \hspace{-1em}\STATE $\{\tau^i\}_{i=1}^N=\mathrm{CFI}(\{h^i\}_{i=1}^N, \model, I(\mu~\rightarrow~\pi^{(\texttt{T})}_{k-1}), \mathcal T, P^{(\texttt{S})})$ 
                
                \STATE \hspace{0.5em}{\color{orange} // Estimate transition stats $\hat P^{(\texttt{T})}\text{ from }\{\tau^i\}_{i=1}^N$}
                \hspace{-1em}\STATE $\tilde P^{(\texttt{T})} = \frac{1}{Z_{\mathrm{T}}}\left(\eta \ P^{(\texttt{T})} + (1-\eta) \ \hat P^{(\texttt{T})}\right)$ 
                \STATE \hspace{0.5em} {\color{algblue}// Regularized policy iteration with $\tilde P^{(\texttt{T})}$}
                \hspace{-1em}\STATE $\pi^{(\texttt{T})}_{k}\leftarrow \mathrm{RegPI}(\pi^{(\texttt{T})}_{k-1}, \gamma, \tilde P^{(\texttt{T})}, \pi^{(\texttt{S})}, \lambda)$
            \ENDFOR
            }
    \end{algorithmic}
\end{algorithm}

We introduce counterfactually augmented policy iteration (CF-PI), the core method of our proposed CFPT framework, the major components of which have been outlined in the previous two subsections.
CF-PI is visualized in Figure~\ref{fig:procedure_overview} and outlined in Alg.~\ref{alg:CF_PI}. When learning a policy in $\ttT$, where a limited number of trajectories {\small $\mathcal{H}^{(\ttT)}$} have been collected with an unknown behavior policy $\mu^{(\ttT)}$, we assume access to an optimal policy distribution {\small $\pi^{(\ttS)}$} as well as transition statistics {\small $P^{(\ttS)}$} from a relevant source domain. In practice {\small $P^{(\ttS)}$} may correspond to expected patient physiological responses to treatment while {\small $\pi^{(\ttS)}$} reflects known treatment protocols.

CF-PI is performed over $K$ iterations where, in each iteration, a batch of counterfactual trajectories {\small $\{\tau^i\}$} from $\ttT$  (Sec.~\ref{sec:aug_gumbel})---sampled according to the current policy {\small $\pi^{(\ttT)}_{k-1}$}---are used to augment the transition statistics {\small $P^{(\ttT)}$}. This augmentation {\small (Alg.~\ref{alg:CF_PI}, line 5)} is a re-normalized weighted sum between the observed {\small $P^{(\ttT)}$} and {\small $\hat P^{(\ttT)}$} estimated from $\{\tau^i\}$. The parameter $\eta$ is empirically chosen (see Sec.~\ref{sec:app_eta_analysis} in the Appendix) to heavily favor observed transition statistics while still incorporating added diversity through counterfactual sampling. {\small $Z_T$} is the normalizing constant over all successor states {\small $S'=s'$} from any given state $s$. 
{\small $P^{(\ttT)}$} is then used in regularized Policy Iteration ($\mathrm{RegPI}$, Sec.~\ref{sec:reg_pi}) to update the policy {\small $\pi_k^{(\ttT)}$}.  
$\mathrm{RegPI}$ is run to convergence or for a set number of iterations. The resulting policy {\small $\pi^{(\ttT)}_k$} is used to sample additional counterfactual trajectories at the beginning of the next iteration.

We describe the full CFPT procedure in extensive detail (including complete psuedocode) in the Appendix,  Sec.~\ref{sec:app_CFPT}.

\section{Experimental Setup}\label{sec:experiments}

We demonstrate the benefits of CFPT through a simulated task of providing treatment to septic patients~\citep{oberst2019counterfactual}. We construct domain shift in the simulator by varying the proportions of diabetic patients between $\ttS$ and $\ttT$. Diabetic patients are more challenging to treat due to increased stochasticity in their glucose levels following treatment. 
Discharge (reward of $+1$) occurs when all vitals are `normal' and treatment is discontinued; death (reward of $-1$) occurs if any three of the vitals are simultaneously not `normal'.

{\bf Baselines:}\label{sec:baselines}

Since generalization is guaranteed when all confounding is observed~\citep{wen2014robust}, we hide diabetes status, inducing unobserved confounding. 
This mimics realistic clinical settings where relevant information may not be immediately available. We intend to verify the robustness of CFPT in the target domain $\ttT$ even when the regularization procedure is not guaranteed to be counterfactually stable (i.e. in the presence of unobserved confounding). We compare to several baselines: \begin{inparaenum} \item[i)] {\bf{{\small \textsc{Scratch}}}}, the policy {\small $\pi^{(\ttT)}$} is learned using policy iteration (PI) solely from the observed trajectories {\small $\mathcal{H}^{(\ttT)}$}. \item[ii)] {\bf{{\small \textsc{Pooled}}}} pools the observed {\small $\mathcal{H}^{(\ttS)}$} and {\small $\mathcal{H}^{(\ttT)}$} to learn {\small $\pi^{(\ttT)}$}, analogous to naive regularization of $P^{(\ttT)}$ by pooling data. \item[iii)] {\bf{{\small \textsc{Blind}}}} applies $\pi^{(\ttS)}$ in $\ttT$ without adaptation\end{inparaenum}. 

We also compare CFPT to two ablations showcasing the benefits of each contribution outlined in Sec.~\ref{sec:policy_transfer}. \begin{inparaenum} \item[iv)] {\bf{{\small \textsc{RegPI}}}} omits counterfactual trajectory sampling, only regularizing $\pi^{(\ttT)}$ by $\pi^{(\ttS)}$ (cf. Sec.~\ref{sec:reg_pi}), which is functionally equivalent to the tabular setting of CRR~\citep{wang2020critic}. 
\item[v)] {\bf{{\small \textsc{Red. CFPT}}}} is a reduced form of CFPT where we omit the informative prior from $\ttS$ 
when sampling counterfactual trajectories. Here, the counterfactual trajectories are drawn according to the Gumbel variables from $\ttT$ only. Policy learning is then completed with $\mathrm{RegPI}$. \end{inparaenum}
All settings used to train these policies are included in the Appendix, Sec.~\ref{sec:app_experiments}.

{\bf Setup:}
The behavior policy $\mu$ was found using PI with full access to the MDP (including diabetes state) to provide a strong observation policy, following~\citep{oberst2019counterfactual}. When generating the observed trajectories $\mathcal{H}$, the policy takes random actions w.p. $0.15$ to introduce variation. Within $\ttS, \ |\mathcal{H}^{(\ttS)}|=10000$, with at most $20$ steps per trajectory, where the probability of a trajectory coming from a diabetic patient is $0.1$.
We limit $|\mathcal{H}^{(\ttT)}| = 2000$ and shift the patient distribution to include a varying proportion of diabetic trajectories in range [0.0, 1.0] in 0.1 increments. 

To avoid extrapolation error~\citep{fujimoto2019off} when learning $\pi^{(\ttT)}$, all estimated transition statistics corresponding to actions not found in the data are zeroed out and the empirical transition matrix is renormalized. Further, the PI procedure is penalized when unsupported actions are taken; the trajectory is terminated and a negative reward is returned.

We evaluate the performance of CFPT when:
{\small
\begin{itemize}[leftmargin=*]
    \item Varying the amount of domain shift in $\ttT$, demonstrating that CFPT performs well relative to the baselines even as the patient distribution in $\ttT$ is shifted farther from $\ttS$.
    \item Varying the size of the patient cohort in $\ttT$, evaluating how data-scarcity affects the observed benefit of CFPT. 
\end{itemize}
}

\section{Results}
\label{sec:results}
\subsection{CFPT is Robust Under Domain Shift} 

We evaluate CFPT and the baselines defined on several settings of $\ttT$ where the proportion of trajectories gathered from of diabetic patients in {\small $\mathcal{H}^{(\ttT)}$} is increased in increments of $0.1$.
The prevalence of diabetic patients and with few trajectories, the estimated transition statistics $P^{(\ttT)}$ are far from the truth. This provides an opportunity to demonstrate the benefits of careful transfer from the source domain $\ttS$.

\subsubsection{Robustness of CFPT improvement}

\begin{figure}[!t]
    \centering
    \includegraphics[width=0.95\linewidth]{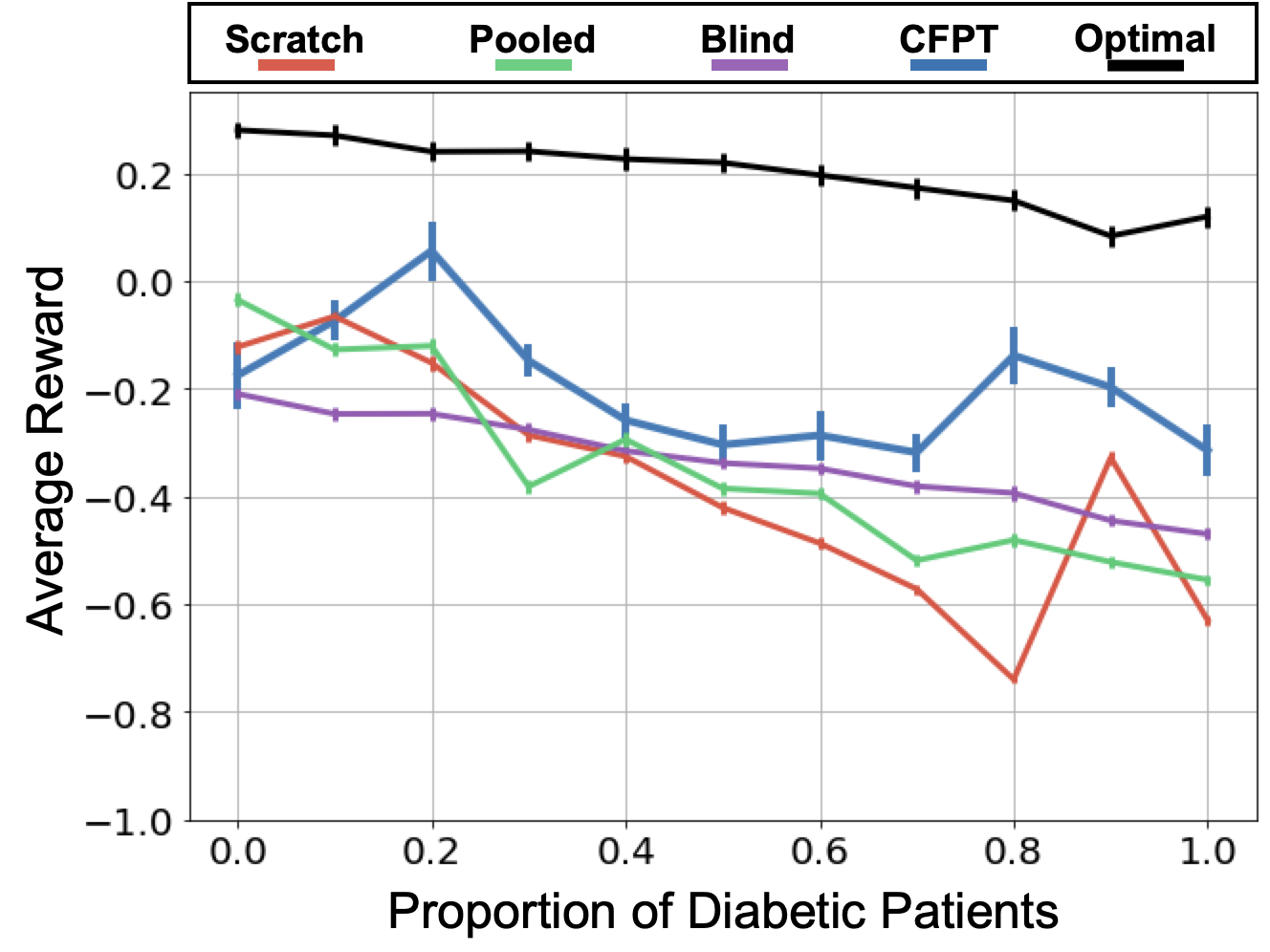}
    \caption{\small Comparison of CFPT with the defined baselines when varying the proportion of diabetic patients in $\ttT$. The black line denotes the observed optimal behavior with full knowledge.} 
    \label{fig:robustness_results}
\end{figure}

Figure~\ref{fig:robustness_results} shows the average reward when applying the learned $\pi^{(\ttT)}$ to simulate an additional $5000$ trajectories across the various shifts in patient population in $\ttT$. The performance of $\pi^{(\ttT)}$ learned with the various baseline strategies is presented alongside the observed optimal behavior $\mu^{(\ttT)}$ (with full knowledge of patient state) as the solid black line. The benefits of CFPT (in blue) are clear across all levels of domain shift, with significant performance improvement when mixture populations in $\ttT$ are the furthest from $\ttS$.

Diabetic patients are harder to treat in this simulator, resulting in a decreasing trend in average reward as the proportion of diabetic patient trajectories increases. For CFPT, the advantages of leveraging $\pi^{\ttS}$ in a domain distributionally similar to $\ttS$ (pDiab$=0.3$) are clear. However, in domains $\ttT$ where the patient distribution is shifted far from $\ttS$, CFPT achieves similar policy improvements. 
The clearest advantage of CFPT is when $\ttT$ has a majority of diabetic patient trajectories. This demonstrates that the causal framework and use of counterfactual regularization provide significant benefits when transferring from $\ttS$. 
Overall, this quantitative evaluation is a strong indication of the benefits of our proposed two-fold regularization when faced with domain shift between domains.

\subsubsection{Ablation Study for CFPT}

\begin{figure}[!t]
    \centering
    \includegraphics[width=\linewidth]{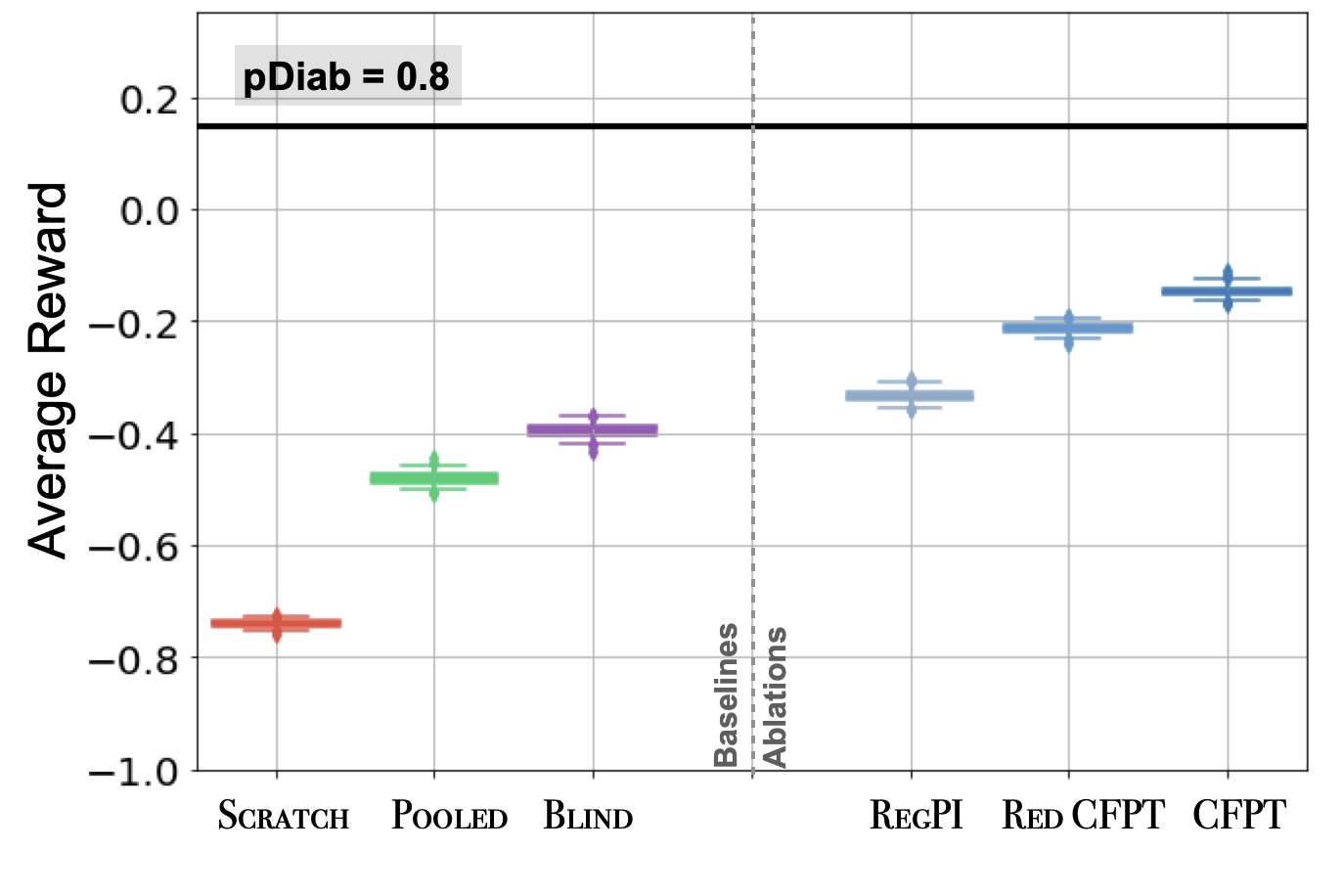}
    \vspace{-0.5cm}
    \caption{\small Comparison of estimated reward in $\ttT$ between CFPT, baselines and, ablations as outlined in Sec.~\ref{sec:baselines}. 95\% uncertainty intervals are found via 100 bootstrapped samples of the 5000 trajectories generated with the learned policy $\pi^{(\ttT)}$.}
    \label{fig:comp_rl_results}
\end{figure}

In Figure~\ref{fig:comp_rl_results} we view the performance of CFPT, the baselines, and ablations in a setting of $\ttT$ with a $0.8$ proportion of diabetic patient trajectories.
The ${\bf{\small \textsc{Pooled}}}$ and ${\bf{\small \textsc{Blind}}}$ baselines provide significant improvements over ${\bf{\small \textsc{Scratch}}}$. With each additional contribution we make in the development of CFPT ${\bf{\small (\textsc{RegPI}} \rightarrow {\small \textsc{Red. CFPT}} \rightarrow {\small \textsc{CFPT})}}$ policy performance steadily improves and approaches the observed return of the optimal behavior policy derived with complete knowledge of MDP and patient diabetes state.

Recall that the {\small \textsc{RegPI}} ablation is functionally equivalent to the recent state of the art offline RL method CRR~\citep{wang2020critic}. The observed improvement over this algorithmic approach demonstrates the value of our proposed regularization for counterfactual trajectory sampling. This further validates the use of causal mechanisms when constructing a transfer approach for offline RL settings. 

\subsubsection{Off-policy Evaluation}

The off-policy evaluation (OPE) of policies learned from fixed data, without the ability to independently test them is a challenging part of offline RL, and has been understudied in partially observed settings~\citep{tennenholtz2020off}. Importance Sampling (IS) can provide an estimate of policy performance with low bias for OPE~\citep{thomas2015safe}, which is desirable in a transfer setting. While we focus on true rewards as our primary evaluation in this paper, we provide OPE estimates in this section for completeness. For this, we use weighted importance sampling (\ope)~\citep{mahmood2014weighted} to evaluate our transfer policies due to its consistency properties. In the Appendix, Sec.~\ref{sec:app_policy_introspection} we use a counterfacutally determined OPE method, CF-PE~\citep{oberst2019counterfactual}, to qualitatively evaluate learned policies. 

OPE estimates generally exhibit significant overconfidence in expected rewards, as areas of high reward are erroneously extrapolated over unseen regions of the state space. We report the results of evaluating \ope~for the learned policies $\pi^{(\ttT)}$ in Table~\ref{tab:ablation_results} including comparisons to learning a policy in $\ttT$ where the diabetic status is known (``Full Obs.''), through Behavior Cloning (``BC'') and what the observed reward of the behavior policy $\mu^{(\ttT)}$. These results are provided for the setting of $\ttT$ with a 0.8 proportion of diabetic patient trajectories. As expected, WIS overestimates the true RL return in $\ttT$, even with poor policies (i.e. {\bf{\small{\textsc{Scratch}}}}). However, we see some semblance of improvement with each component of our proposed CFPT approach. However, the unreliability of these OPE estimates make it difficult to truly evaluate the benefits of transfer with counterfactual regularization.

\begin{table}[ht!]
\centering \small
\begin{tabular}{ccc}
\toprule
\bf{Approach}   & \bf{True RL Reward}      & \bf{WIS Reward}        \\ \hline
Scratch  & $-0.7398 \pm 0.007$ & $0.6388 \pm 0.584$  \\
Pooled   & $-0.4808 \pm 0.012$ & $0.9782 \pm 0.004$  \\
Blind    & $-0.3915 \pm 0.013$ & $0.5874 \pm 0.113$ \\ \hline
RegPI    & $-0.3366 \pm 0.012$ & $0.6266 \pm 0.057$  \\
Red. CFPT & $-0.2116 \pm 0.010$ & $0.7689 \pm 0.077$  \\
CFPT      & $-0.1491 \pm 0.011$ & $0.7333 \pm 0.004$ \\ \hline
Full Obs. & {\bf $-0.0877 \pm 0.012$} & $0.9037 \pm 0.054$ \\
BC & $-0.2078\pm0.0109$ & $0.9836 \pm 0.002$\\ 
Obs. $\mu^{(\ttT)}$ & $0.1486 \pm 0.018$ & -- \\
\bottomrule
\end{tabular}
\caption{\small Numerical values corresponding the policy performance results presented in Figure~\ref{fig:comp_rl_results}.}
\label{tab:ablation_results}
\end{table}

Fortunately, CF-PE allows for the comparison of individual counterfactual trajectories influenced by CFPT and other methods. This form of introspective evaluation can help identify glaring safety issues for deployment of a trained policy in a new domain. As seen in the Appendix, Sec.~\ref{sec:app_policy_introspection}, CFPT acts more conservatively and closely approximates the observed behavior, leading to more stable performance.

\subsection{CFPT Demonstrates Improvement Among Various Levels of Data-Scarcity in $\ttT$}

\begin{figure}[!ht]
    \centering
    \includegraphics[width=0.9\linewidth]{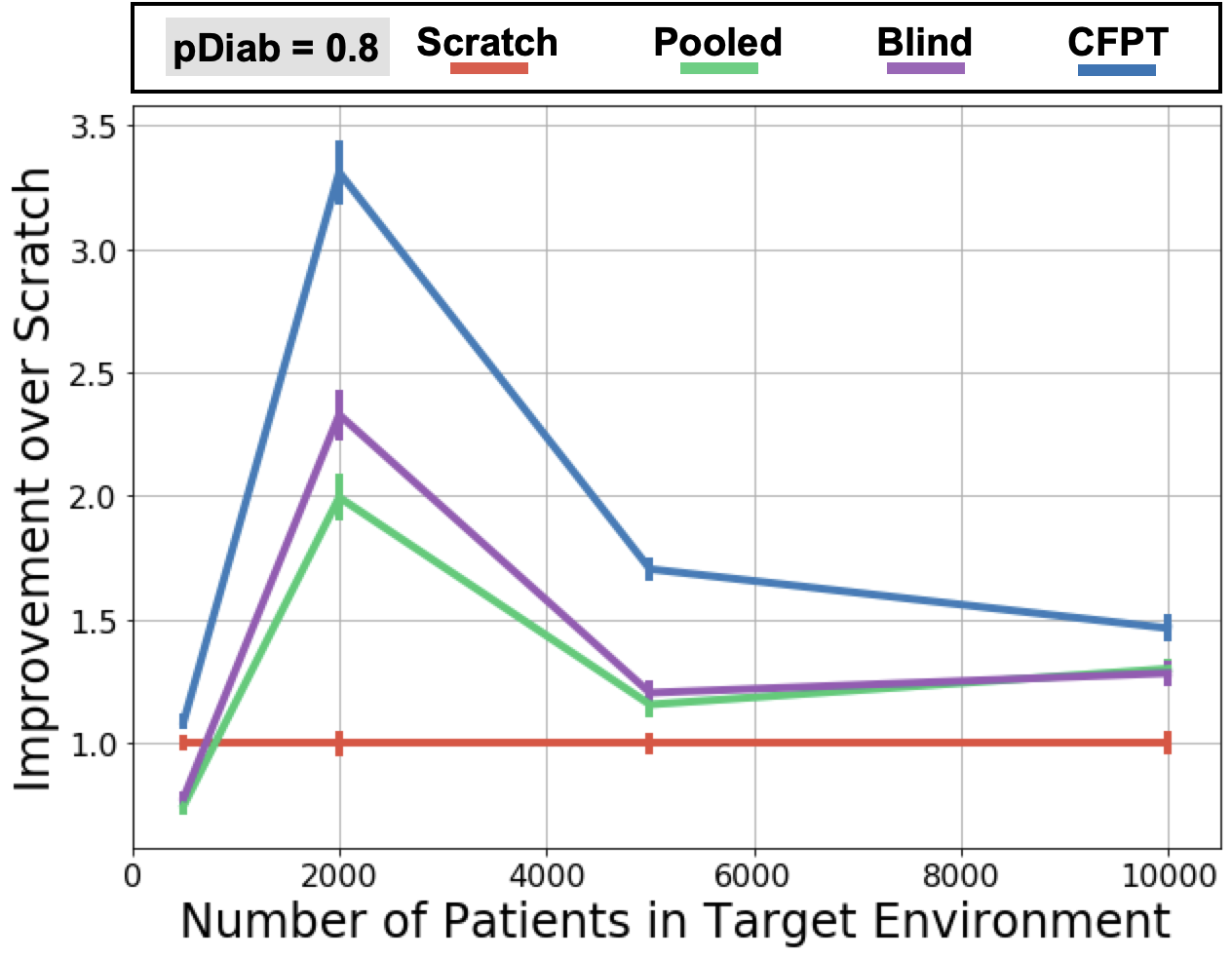}
    \caption{\small Performance improvement via transfer approaches over the naive {\footnotesize \textsc{Scratch}} baseline with respect to the number of trajectories available in $\ttT$. CFPT provides significant improvement when the size of the target domain is small relative to source domain.}
    \label{fig:improvement_by_target_size}
\end{figure}

The benefits of transfer may vary as more or less data is available in the observed $\mathcal{H}^{(\ttT)}$. Characterizing this benefit can aid understanding of the levels of regularization one should use for transferring from $\ttS$. This is particularly important when $\ttT$ may feature a significantly shifted data distribution, as we have simulated in this paper. Figure~\ref{fig:improvement_by_target_size} demonstrates the improvement of different transfer approaches over a ${\bf{\small{\textsc{Scratch}}}}$ policy as the size of $\mathcal{H}^{(\ttT)}$ changes. We evaluate the effects of transfer when {\small $|\mathcal{H}^{(\ttT)}|\in\{500, 2000, 5000, 10000\}$} with $\text{pDiab}=0.8$. When very few samples are available, transfer does not reliably improve over ${\bf{\small{\textsc{Scratch}}}}$, since there is little data to refine $\pi^{(\ttS)}$ with. As more samples are available, clear benefits are observed from transfer, with more than a 3x improvement when using CFPT.
These benefits diminish as more data is available in $\ttT$, allowing for an effective policy to be learned natively. We further analyze these policy improvements for the diabetic and non-diabetic sub-populations of $\ttT$ in the Appendix, Sec.~\ref{sec:app_subpopulation_analysis}.

\subsection{Quality of Counterfactual Samples}
\label{sec:counterfactual_samples}

We chose to use the Gumbel-Max SCM to initiate this version of our CFPT framework because it guarantees that counterfactual samples will lie within the support of $\mathcal{H}^{(\ttT)}$. It is not merely a qualitative formulation, as it provides stable sampling characteristics.         
We quantify the quality of these counterfactual samples by comparing i) target domain samples (collected with unknown $\mu^{(T)}$) and ii) target domain counterfactual samples using $P^{(S)}$ as a prior. In Figure~\ref{fig:maskedDiab} we compare the features when diabetes status is unobserved (a corresponding analysis when the diabetes status \emph{is} observed can be found in the Appendix, Sec.~\ref{sec:app_counterfactual_samples}). The counterfactually sampled data (on right) provides better coverage of the features while also not overly reducing the relative balance within the distribution of each feature. This helps to confirm the validity of using the counterfactually sampled trajectories when improving the robustness of the learned $\pi^{(\ttT)}$.

\begin{figure}[h!]
    \centering
    \includegraphics[width=0.9\linewidth]{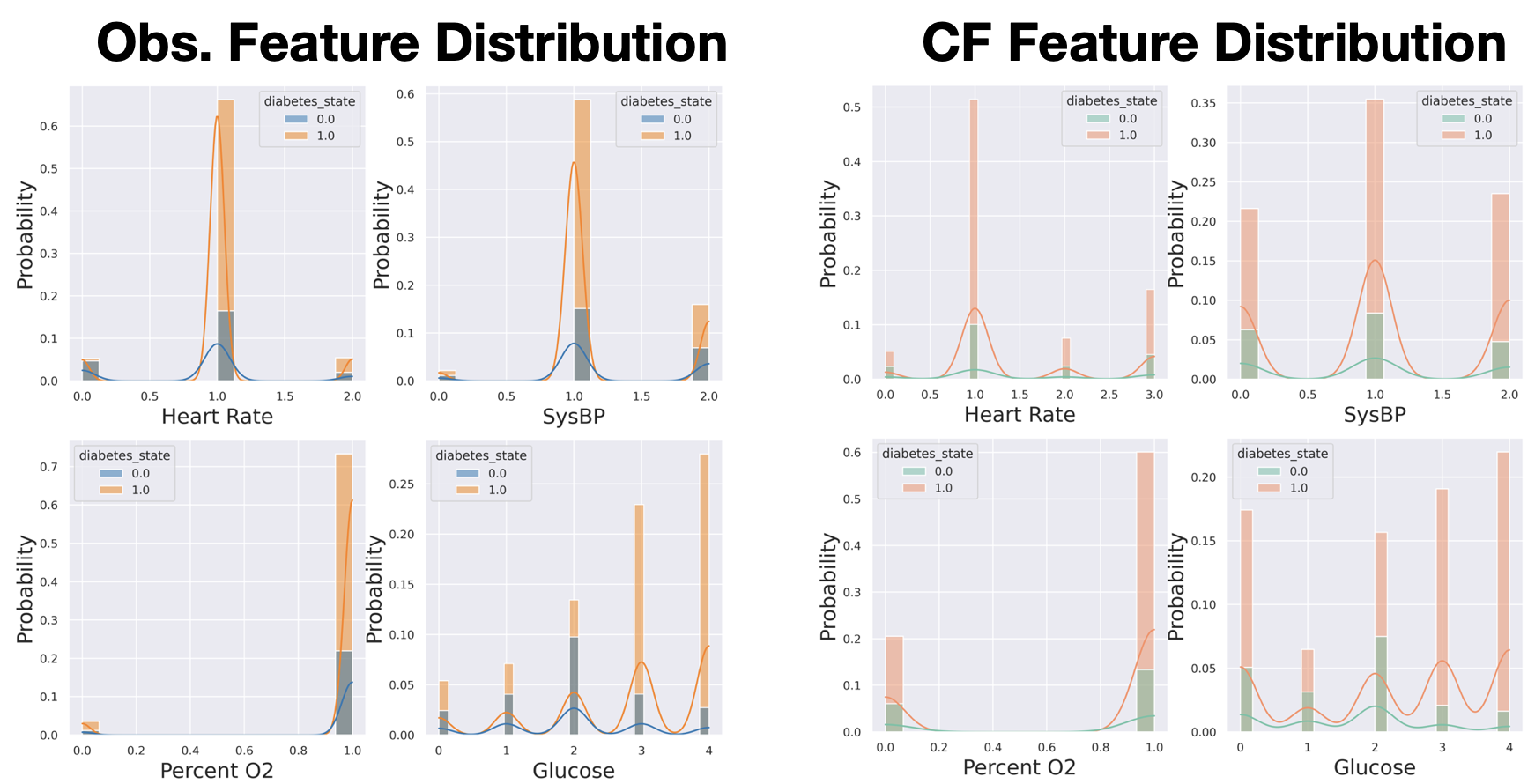}
    \caption{Feature distributions of the observed data (on left) and counterfactual samples (on right) with unobserved confounding}
    \label{fig:maskedDiab}
\end{figure}

\section{Conclusion}
Motivated by challenges of policy transfer in offline, off-policy clinical settings, we have introduced Counterfactually Guided Policy Transfer. This procedure leverages complementary elements of a data-rich source domain $\ttS$ to facilitate better learning in a data-scarce target domain $\ttT$. In our transfer framework we utilize: \begin{inparaenum} \item[1)] The observed transition statistics {\small $P^{(\ttS)}$} and \item[2)] the trained treatment policy {\small $\pi^{(\ttS)}$} to guide development of an effective policy {\small $\pi^{(\ttT)}$}\end{inparaenum}. 
By carefully designing transfer policies under restricted settings 
between domains we provide a principled justification for both the counterfactual and policy regularization frameworks we propose. 
In clinical practice, {\small $P^{(\ttS)}$} may correspond to expected patient physiological responses to treatment while {\small $\pi^{(\ttS)}$} reflects known treatment protocols. Both these elements can be feasibly shared in a secure manner and, as demonstrated by this work, used to improve treatment policy development.

In future work, we plan to adjust the regularization policies adaptively, based on the uncertainty of the transition statistics and treatment selection process. The work we have presented in this paper stands as an initial step in the development of counterfactually-aided policy transfer to reliably extend learned models beyond the domain they were trained in. While the discrete setting we have used in this work is suitable for a proof of concept, we intend to broaden the theoretical foundation supporting our procedure to admit continuous state spaces and treatments. This will support policy development using retrospective data derived from electronic medical records, moving us one step closer toward positively contributing to clinical practice.

\section*{Institutional Review Board (IRB)}
The research presented in this paper provides a proof of concept for a novel policy transfer method and is validated on simulated data. As such this research does not require IRB approval.

\acks{We thank the anonymous reviewers and our many colleagues who contributed to thoughtful discussions and provided timely advice to improve this work. We specifically appreciate the feedback provided by Sindhu Gowda, Sana Tonekaboni, Chun-Hao Chang, Elliot Creager, David Madras, Vinith Suriyakumar and Nathan Ng.

This research was supported in part by Microsoft Research, a CIFAR AI Chair at the Vector Institute, a Canada Research Council Chair, and an NSERC Discovery Grant.

Resources used in preparing this research were provided, in part, by the Province of Ontario, the Government of Canada through CIFAR, and companies sponsoring the Vector Institute \url{www.vectorinstitute.ai/\#partners}.}

\bibliography{refs}

\appendix

\section{Background: SCM}\label{app:scm}

\subsection{Structural Causal Models~\cite{pearl2009causality}}\label{sec:scm}
A structural causal model $\cM$ describes the causal mechanisms driving a system. It consists of an ordered triple $\langle \bU, \bX, \bF\rangle$; a set of independent exogenous random variables $\mathbf{U} = \{U_1, U_2, \ldots, U_k\}$ that represent factors of variation outside the model, $\bX$ comprises the endogenous variables modeled in the causal system and, the set of functions $\bF$ defined by $X_{i} \coloneqq f_{i}(\textbf{PA}_i, U_i) \ \forall i$ where $\textbf{PA}_i \subseteq \bX \setminus X_i$ govern the causal mechanisms. $\textbf{PA}_i$ are the parents of $X_i$ in a causal DAG $\cG$. The framework attributes probabilistic Markov assumptions to the joint distribution $P^{\cM}$associated with the variables $(\bX, \bU)$ in the graph.  This characterizes a probability distribution, implying that one can observe samples true to the underlying causal graph and mechanism. 

\begin{definition}{Interventional Distribution:}
An intervention $I$ in an SCM $\cM$ consists of replacing some functions $f_{i}(\textbf{PA}_i, U_i)$ with a different governing causal mechanism $f_{i}^I(\textbf{PA}_i^I, U_i)$ where $\textbf{PA}_i^I$ are the parents of $X_i$ in a new DAG $\cG^I$. Note that the interventional distribution does not change the exogenous mechanisms driving the system. The resulting SCM, denoted by $\cM^{do(I)}$ has a new joint distribution denoted by $P^{\cM^{do(I)}}$.
\end{definition}

An intervention $I$ is generally used to evaluate the \emph{prospective} effect of perturbing the underlying causal mechanism. A more useful quantity in off-policy learning is the \emph{counterfactual} which allows you to answer the causal queries of the form: ``what would have happened had we given the patient medication $b$ having observed no improvement with medication $a$?'' Answering such \emph{retrospective} queries requires inferring a model of the exogenous variables $P(\bU | \bX = \bx)$ and intervene with $I$ on a causal system with exogenous noise priors $p(\bU)$ replaced by $p(\bU | \bX=\bx)$.

\begin{definition}{Counterfactual Distribution:}
Let $\cM^{\bx}$ correspond to the SCM where the exogenous noise model $p(\bU)$ in $\cM$ is replaced by $p(\bU | \bX = \bx)$. Intervening with $I$ on the resulting SCM $\cM^{\bx}$ yields a new SCM $\cM^{do(I)|\bx}$ and induces the joint counterfactual distribution $P^{\cM^{do(I)| \bX = \bx}}$.
\end{definition}

\subsubsection{Connections between expected counterfactual reward and ACE/ATE}
\label{sec:app_ace_connections}

Naturally, to determine if a policy is better than a behavior policy $\mu$, the quantity of interest is the difference in expected rewards between the behavior policy $\mu$ and another policy $\pi$. In causal inference literature, this is analogous to evaluating average treatment effect (ATE) under \emph{soft} interventions in the underlying causal model. In our case this is a POMDP represented as a Structural Causal Model (SCM). Specifically, {\small $\mathrm{ATE}_\pi = E[\cR(\tau) | do(\pi)] - E[\cR(\tau) |do(\mu)]$} a quantity that can be interpreted as an outcome in the SCM. Note again that in off-policy settings, the first expectation term is obtained under the distribution $P^{do(I'(\mu \rightarrow \pi)) | {I(\mu)}}$ i.e. with modified posteriors over exogenous variables.

\subsection{Gumbel-Max SCM~\cite{oberst2019counterfactual}}\label{sec:app_gumbel_max_scm}

\begin{definition}{Gumbel-Max Trick:}
a sampling procedure from any discrete distribution with $k$ categories, parametrized by $p_i = P(X = i), \forall i \in \{1, 2, \ldots, k\}$. First, sample $k$ independent Gumbel variables $g_j$ with location 0, scale 1. Set the sampled outcome $k = \argmax_{j} \log{p_j} + g_j$.
\end{definition}

A Gumbel-Max SCM is one in which all nodes $\bX$ are discrete random variables. Given independent Gumbel variables $\bg = \{g_1, g_2, \ldots, g_k\}$, the causal mechanisms are given by: $X_i \coloneqq f_{i}(\textbf{PA}_i , g_i) = \argmax_{j} \log{p(X_i = j | \textbf{PA}_i)} + g_j$.

Non-identifiability of causal effect estimation under counterfactual scenarios is challenging for reliable transfer. That is, there may be multiple SCMs consistent with observations that provide different counterfactual estimates. In order to reliably draw causal conclusions from a counterfactual query, which is what we will need, further assumptions are required. In the case of binary SCMs, this assumption is given by the monotonicity condition~\cite{pearl2009causality}  and in the discrete case known as counterfactual stability.

Let $P^{do(I)}(Y=i) = p_i \ \forall i \in [L]$ and $P^{do(I')}(Y=i) = p_i' \ \forall i \in [K]$. Let $P^{do(I)}(X=i)$ be the probability of observing $i$ under intervention $I$ for variable $X$ in a discrete SCM and the observed outcome be represented by $X_{I}$. Then $P^{do(I') | X_{I}=i}(X=j)$ is the counterfactual probability of observing outcome $j$ having observed $i$ under intervention $I$.

\subsection{Gumbel-Max Topdown Sampling}\label{app:topdown}

The Gumbel-Max trick enables sampling from categorical distributions $\mathrm{Cat}(\alpha_1,\ldots,\alpha_K)$, where the category $k$ will be selected with probability $\alpha_k$ among K distinct categories~\citep{hazan2012partition,maddison2014sampling,maddison2016concrete}. This sampling procedure rests on inferring Gumbel variables $g_k$ that can be transformed into these probabilities $\alpha_k$.

The density of Gumbel variables with location parameter $\log{\alpha_k}$ and scale $1$ is:

{\small
\begin{align}\label{eq:density_gumbels}
\begin{split}
    f_{\log{\alpha_k}} (g_k) &= \exp{(-g_k + \log{\alpha_k})}\exp{(-\exp{(-g_k + \log{\alpha_k})})} \\
    & = \exp{(-g_k + \log{\alpha_k})}F_{\log{\alpha_k}}(g_k),
\end{split}
\end{align}
}

where $F_{\log{\alpha_k}}(g_k)$ is the CDF of the Gumbel variable $g_k$. 
Without any prior on the Gumbel variables $g^{(\ttT)}_k$, corresponding to the discrete patient states observed in a target domain $\ttT$, the location parameters can be obtained according to the empirical transition probabilities $P^{(\ttT)}$. That is, $p(\balpha) = \delta(\log{P^{\ttT}}(\cdot | \cdot))$ where $\delta$ is the dirac-delta distribution. Sampling from this Gumbel given observation $k'$ can be done using the Topdown procedure from~\cite{maddison2014sampling}.

Now consider the joint distribution of $k'$ and $\bg^{(\ttT)}$ for any fixed state-action pair (we drop explicit notation for clarity). To account for the informative prior $P^{(\ttS)}$, we treat the locations of these Gumbel variables to be random-variables $\balpha$. To obtain the joint distribution, we integrate over $\balpha$:
{\small
\begin{align}
    \begin{split}\label{eq:post}
        p(k', g^{(\ttT)}_1, \ldots, g^{(\ttT)}_n) &= \\ 
        \int_{\balpha} \frac{\alpha_{k'}}{Z} f_{\log{Z}}(g^{(\ttT)}_{k'}) &\prod_{i \neq k'} \bigg[  f_{\log{\alpha_i}}(g^{(\ttT)}_i) \frac{\llbracket g^{(\ttT)}_{k'} \geq g^{(\ttT)}_i\rrbracket}{F_{\log{\alpha_i}}(g^{(\ttT)}_{k'})} \bigg] p(\balpha) d\balpha 
    \end{split}
\end{align}
}

Equation (\ref{eq:post}) can be obtained exactly following\footnote{\url{https://cmaddis.github.io/gumbel-machinery}}~\cite{maddison2014sampling}.
That is, for a fixed and known $\alpha_k$, the Gumbel corresponding to the observed outcome $k'$, i.e. $g^{(\ttT)}_{k'}$ is itself a Gumbel variable with location parameter $Z = \log{\sum_{k=1}^K \alpha_{k}}$. It follows that the maximum value $k'$ and corresponding Gumbels are independent and the rest of the exogenous variables $g^{(\ttT)}_k~\forall k \neq k'$ are truncated by this maximum value corresponding to $k'$. To leverage information from the source domain $\ttS$, we replace the dirac-delta prior by a mixture of the source and target transition statistics (see Equation~\ref{eq:mixture_prior} in Sec.~\ref{sec:aug_gumbel}). The sampling procedure follows a modified top-down procedure such that for every counterfactual sample, we first select the mixture component with probability $[w^{(\ttT)}, 1-w^{(\ttT)}]$, followed by posterior sampling over the Gumbels.

\subsection{Mixture-prior preserves counterfactual stability}\label{sec:app_stability}

\begin{definition}{Counterfactual Stability:}
An SCM over discrete random variables is counterfactually stable if: If we observe $X_I = i$, then $\forall j \neq i$, if $\frac{p_i'}{p_i} \geq \frac{p_j'}{p_j}$, implies that $P^{do(I') | X_{I}=i}(X=j) = 0$.
\end{definition}

Our proof is based on the insight that counterfactual stability is invariant to choice of prior so long as the gumbel samples are fixed across interventions. Our modified topdown sampling procedure ensures the same gumbel samples are used across interventions. Hence we preserve counterfactual stability even with regularization. For completeness, we include the contrapositive proof of~\citet{oberst2019counterfactual} here:

As denoted before, let $X^{(\ttT)}_{I} = i$ (we drop $\ttT$ from superscript for random variables when context is clear) be the outcome observed under intervention (behaviour policy) in the target domain. The state observation $i$ implies almost surely:
\begin{equation}\label{eq:cs_1}
    \log{p}_{i} + g^{(\ttT)}_{i} > \log{p}_{j} + g^{(\ttT)}_{j} \forall j \neq i
\end{equation}

where $p_i \coloneqq P^{(\ttT)}(X = i)$ is short hand for the state-transition probabilities in the target domain induced using the Mixture-prior described above. To prove counterfactual stability, the contrapositive is proved i.e. $\forall j \neq i, P^{do(I')|X_{I}=i}(X = j) \neq 0 \implies \frac{p_i'}{p_i} < \frac{p_j'}{p_j}$. 

To begin with, if $P^{do(I')|X^{(\ttT)}_{I}=i}(X^{(\ttT)} = j) \neq 0$ implies that there exist gumbel variables $g^{(\ttT)}_i$ and  $g^{(\ttT)}_j$ such that:
\begin{equation}\label{eq:cs_2}
    \log{p}_{i}' + g^{(\ttT)}_{i}  < \log{p}_{j}' + g^{(\ttT)}_{j} 
\end{equation}

where $p_j' \coloneqq P^{do(I')|X_{I}=i}(X^{(\ttT)} = j)$. Since gumbels sampled for Equation (\ref{eq:cs_1}) and (\ref{eq:cs_2}) are fixed, there must exist gumbels that satisfy both equations. The only difference is that an informative prior is imposed on these gumbels is different. Thus counterfactual stability is not violated due to the mixture prior and modified gumbel procedure. Combining the inequalities and re-arranging, we establish the contrapositive with regularization.

\section{Estimating counterfactual rewards with informative prior}\label{sec:app_rewards}
Our proof largely follows~\citet{oberst2019counterfactual} and~\citet{buesing2018woulda} although with a different posterior on the Gumbel exogenous variables. We make the difference explicit in the following:
$\mu^{(\ttT)}$ be the behavior policy in the target environment and the corresponding trajectories denoted by $\tau^{\mu^{(\ttT)}}$. Let $\pi^{(\ttT)}$ be a candidate policy for which expected rewards are to be estimated and $\tau^{\pi^{(\ttT)}}$ be the counterfactual trajectories using conditional posteriors $p(\bU^{(\ttT) | \tau})$ over exogenous variables $\bU^{(\ttT)}$. $\tau^{\pi^{(\ttT)}}$ is a deterministic function of $\bU^{(\ttT)}$. The prior distributions over $\bU$ are $p^{\pi}(\bU^{(\ttT)}) = p^{\mu}(\bU^{(\ttT)}) = p(\bU^{(\ttT)})$ (which remains the same as any informative prior coming from the source environment imposed in this framework). We drop the notation $(\ttT)$ in the following as we are only concerned about the target environment hereon. Source distributions, if any, will be made explicit. Expected reward is then given by:
\begin{align}
     E_{p^{\pi}}[\cR(\tau)] &= \int_{\bu} \cR(\tau(\bu)) p^{\pi}(\bu) d\bu \\
     &= \int_{\bu} \cR(\tau(\bu)) p^{\mu}(\bu) d\bu \\
     &= \int_{\bu} \cR(\tau(\bu)) \bigg( \int_{\tau} p^{\mu}(\tau, \bu) d\tau \bigg) d\bu \label{eq:int} \\
     &= \int_{\bu} \cR(\tau(\bu)) \bigg( \int_{\tau} p^{\mu}(\bu | \tau)p^{\mu}(\tau) d\tau \bigg) d\bu \\
      &=  \int_{\tau} \int_{\bu} \cR(\tau(\bu)) p^{\mu}(\bu | \tau)p^{\mu}(\tau) d\bu d\tau   \label{eq:int2}\\
      &= E_{\tau^{\pi} \sim p^{\mu}(\tau)} \bigg[ \int_{\bu} \cR(\tau(\bu)) p^{\mu}(\bu | \tau) d\bu \bigg] \\
      &= E_{\tau^{\pi} \sim p^{\mu}(\tau)} \bigg[ E_{\bu \sim p^{\mu}(\bu | \tau)} [ \cR(\tau(\bu)) ] \bigg]  \label{eq:int3}
\end{align}
Where note that Equation (\ref{eq:int}) integrates over \emph{observed} policies only. This allows to swap integrals in Equation (\ref{eq:int2}). The key difference is that in Equation (\ref{eq:int3}), for the subset of exogenous variables $\bg^{(\ttT)} \subseteq \bu^{(\ttT)}$, the posterior is inferred by incorporating the mixture prior that helps regularize from the source.

\subsection{Justification of Counterfactual Regularization}\label{app:regintuition}
We consider two subpopulation groups (diabetic) and (non-diabetic) and the corresponding transition dynamics $P_{d}(S|S,A)$ and $P_{nd}(S|S,A)$. We justify our counterfactual regularization using two cases i) where diabetes status of the patient is known in both source and target environment, ii) diabetes status is unknown in both source and target. We assume here that the statistical bias in the estimated transition estimates of diabetic patients $\hat{P}^{(\ttS)}_{d}(S|S,A)$ and $\hat{P}^{(\ttT)}_{d}(S|S,A)$ is higher in the source domain than in the target domain (by virtue of number of samples from this sub-population observed in both domains). The effect is the opposite for non-diabetics. i.e. the bias is lower in the source than the target domain. That is:
\begin{align}
    \begin{split}
    \| \hat{P}^{(\ttS)}_{d}(S|S,A) - &P_{d}(S|S,A)\| \\ \geq~\| &\hat{P}^{(\ttT)}_{d}(S|S,A) - P_{d}(S|S,A)\|    
    \end{split}
\end{align}

\begin{align}
    \begin{split}
    \| \hat{P}^{(\ttS)}_{nd}(S|S,A) - &P_{nd}(S|S,A)\| \\ \leq~\| &\hat{P}^{(\ttT)}_{nd}(S|S,A) - P_{nd}(S|S,A)\|
    \end{split}
\end{align}

Under this setting, consider a vanilla regularization in the target-domain for the transition statistics where we use a convex combination of source and transition estimates for each sub-group instead of using the target-domain estimates only (analogously for the non-diabetic subgroup): $\eta \hat{P}^{(\ttS)}_{d}(S|S,A). + (1-\eta) \hat{P}^{(\ttT)}_{d}(S|S,A)$ where $0\leq \eta \leq 1$.

Then the statistical bias for the non-diabetic group is given by:
\begin{align}
\begin{split}
    &\|P_{nd}(S|S,A) - \eta \hat{P}^{(\ttS)}_{nd}(S|S,A) + (1-\eta) \hat{P}^{(\ttT)}_{nd}(S|S,A)\| \\
    = &\|~\eta~(P_{nd}(S|S,A) - \hat{P}^{(\ttS)}_{nd}(S|S,A)) \\ 
    &\qquad + (1-\eta) (P_{nd}(S|S,A) - \hat{P}^{(\ttT)}_{nd}(S|S,A))\| \\ 
     \leq &~\eta~\| \hat{P}^{(\ttS)}_{nd}(S|S,A) - P_{d}(S|S,A)\| \\
    &\qquad + (1-\eta) \| \hat{P}^{(\ttT)}_{nd}(S|S,A) - P_{nd}(S|S,A)\| \\
     \leq &~\eta~\| \hat{P}^{(\ttT)}_{nd}(S|S,A) - P_{nd}(S|S,A)\| \\
     &\qquad + (1-\eta) \| \hat{P}^{(\ttT)}_{nd}(S|S,A) - P_{nd}(S|S,A)\| \\
      = &~\| \hat{P}^{(\ttT)}_{nd}(S|S,A) - P_{nd}(S|S,A)\|
\end{split}
\end{align}

The regularization from the source, done naively, will benefit the non-diabetic group. However this is not necessarily the case for the diabetic group (notice that the bias can demonstrated to be better than the source environment). However, since diabetics are the majority subpopulation in the target, such naive regularization is insufficient. Consider instead the exogenous variables corresponding to the transition dynamics model, specifically the Gumbel variables. The Gumbel variables in the source and the target are essentially parameterized by the $\log{\hat{P}^{(\ttS)}_{d}(S|S,A)}$ and $\log{\hat{P}^{(\ttT)}_{d}(S|S,A)}$ respectively (similarly for the non-diabetic population when the status is known). Intuitively we are essentially replacing the deterministic regularization above with a stochastic one where the so that the sampled Gumbels can still be utilized under the \emph{true} dynamics of the target domain to generate counterfactual trajectories. Thus, our Mixture-top-down sampling can be considered as a variational/stochastic procedure to the naive regularization procedure. Notably, the stochastic procedure \emph{decouples} the transition dynamics regularization into two steps, i) sampling Gumbels with potentially biased transition estimates, and ii) augmenting trajectories according to the true target dynamics that improves statistical estimation of the dynamics in the target. 

These same insights hold true when diabetes status is not known i.e. in the presence of unobserved confounding, except that a cumulative transition statistic is available instead of separate estimates for each sub-population.

\section{KL-aggregation for CF-PI}\label{sec:app_kl_aggregation}

For discrete action space, KL-aggregation for regularization over policy is equivalent to log-aggregation~\cite{heskes1998selecting}. The proof here is provided for completeness. Consider the following aggregation setup over two discrete distributions:
\begin{align}
    \begin{split}
        \pi &= \argmin_{\pi} \lambda \infdiv{\pi}{\nu} + (1-\lambda) \infdiv{\pi}{\pi^{(S)}}  
    \end{split}
\end{align}

This can be posed as a parametric minimization over the vector $\pi \in \Delta^{K-1}$ (where $K$ is the dimensionality of the action space) as follows:
\begin{align}\label{eq:app_kl_agg}
    \begin{split}
        & \argmin_{\pi} \lambda \langle \pi^{T} , \log{\pi} - \log{\nu} \rangle +  \langle \pi^{T} , \log{\pi} - \log{\pi^{\ttS}} \rangle \\
        & \, \text{s. t. } \pi \in \Delta^{K-1} 
    \end{split}
\end{align}

Equation~\ref{eq:app_kl_agg} is convex in $\pi$ with a convex (simplex) constraint. Simply writing out the Lagrangian, provides:
\begin{align}\label{eq:app_kl_agg2}
    \begin{split}
        \argmin_{\pi} ~\lambda~&\langle \pi^{T} , \log{\pi} - \log{\nu} \rangle +  \langle \pi^{T} , \log{\pi} - \log{\pi^{\ttS}} \rangle \\ &\qquad\qquad+~\mu~(\sum_{k=1}^K \pi_k - 1) + \beta \pi \\
        & \, \text{where } \beta \geq 0
    \end{split}
\end{align}

Taking the gradient and setting to $0$ yields:
\begin{equation}
    (1+\log{\pi}) + \mu \mathbf{1} + \beta = \lambda \log{\nu} + (1-\lambda) \log{\pi^{\ttS}} 
    \label{eq:app_kl_agg3}
\end{equation}
If $1 + \mu \mathbf{1} + \beta = 0$, then $\log{\pi} =  \lambda \log{\nu} + (1-\lambda) \log{\pi^{\ttS}}$ and the simplex constraint is satisfied.

\section{CFPT Procedure}\label{sec:app_CFPT}
\begin{algorithm*}[ht] \small
    \caption{Counterfactually Guided Policy Transfer}
    \label{alg:CFPT}
    \begin{algorithmic}[1] 
        {\small
        
        \STATE {\color{alggreen} // Counterfactual inference (CFI) with source environment prior}
        \STATE \textbf{\textsc{CFI}}(data $\hat x_o$, SCM $\model$, intervention $I$, query $X_q$, prior $X^{(\texttt{S})}_P$)
            \STATE $\hat u \sim p(u|\hat x_o)$ \COMMENT{Sample noise variables from posterior over latent parameters} 
            \STATE $p(u)\leftarrow \delta (u-\hat u)$ \COMMENT{Replace noise distribution in $p$ with $\hat u$}
            \STATE $f_i\leftarrow f_i^I$ \COMMENT{Perform intervention $I$}
            \STATE \textbf{return} $x_q\sim p^{\doo(I)}(x_q\vert \hat u)$ \COMMENT{Simulate from the counterfactual posterior over model $\model^I_{\hat x_o}$, Alg.~\ref{alg:updated_topdown}}
        \newline
        
        \STATE {\color{algblue}// Regularized Policy Iteration (RegPI) }
        \STATE \textbf{\textsc{RegPI}}(current policy $\pi^{(\texttt{T})}$, discount $\gamma$, aug. statistics $\tilde P^{(\texttt{T})}$, source policy $\pi^{(\texttt{S})}$, reg. param $\lambda$)
        \STATE Initialize $V(s)$ for all $s\in \cS$
            \REPEAT
            \REPEAT
                \FOR{each $s\in\cS$}
                    \STATE $v\leftarrow V(s)$
                    \STATE $V(s) \leftarrow \sum_{s'}\tilde P^{(\texttt{T})}(s'|s,\pi^{(\texttt{T})}(s))\left[\mathcal{R}(s,\pi^{(\texttt{T})}(s)) + \gamma V(s')\right]$
                \ENDFOR
            \UNTIL{convergence}
                \FOR{each $s\in\cS$}
                    \STATE $\nu(\cdot|s) \leftarrow \frac{1}{Z_p}\left(\mathbf{\mathcal{R}}(s,\cdot) + \gamma \left(\tilde{P}^{(\texttt{T})}(s'|s,\cdot)\otimes \mathbf{V}(s')\right)\right)$ \COMMENT{Gen. a proposal dist. over actions}
                    \STATE $\pi^{(\texttt{T})}(s)\leftarrow \argmax_a \exp\left\{\lambda\log\nu(a|s)+(1-\lambda)\log\pi^{(\texttt{S})}(a|s)\right\}$ \COMMENT{KL minimization, Eq.~\ref{eq:app_kl_agg3}}
                \ENDFOR
            \UNTIL{ $\pi^{(\texttt{T})}$ converges or after MAX_ITERATIONS}
        \newline
        \STATE {\color{darkgray}// Counterfactual Policy Iteration (CF-PI)}  
        \STATE \textbf{\textsc{CF-PI}}(SCM $\model$, init. policy $\pi^{(\texttt{T})}_0$, source policy $\pi^{(\texttt{S})}$, source statistics $P^{(\texttt{S})}$, num. iters $K$, num. traj samples $N$, mixture param $\eta$)
            \FOR{$k=1,\ldots, K$}        
                \STATE \hspace{2em} {\color{darkgray}// Gather a batch of counterfactually generated trajectories in the target environment}
                \STATE $\{h^i\}_{i=1}^N\sim\mathcal H^{(\texttt{T})} \subset \mathcal T$ \COMMENT{Sample batch of trajectories from observed data}
                \STATE $\{\tau^i\}_{i=1}^N=\mathrm{CFI}(\{h^i\}_{i=1}^N, \model, I(\mu\rightarrow \pi^{(\texttt{T})}_{k-1}), \mathcal T, P^{(\texttt{S})})$ \COMMENT{{\color{alggreen}Counterfactual rollouts under $\pi^{(\texttt{T})}_{k-1}$}}
                
                \STATE \hspace{1em} {\color{orange}// Estimate empirical transition statistics $\hat P^{(\texttt{T})}\text{ from }\{\tau^i\}_{i=1}^N$}
                \STATE $\tilde P^{(\texttt{T})} = \frac{1}{Z_{\mathrm{T}}}\left(\eta \ P^{(\texttt{T})} + (1-\eta) \ \hat P^{(\texttt{T})}\right)$ \COMMENT{Augment observed environment transition statistics}
                \STATE \hspace{1em} {\color{algblue}// Regularized policy iteration with counterfactually augmented target env. transition statistics}
                \STATE $\pi^{(\texttt{T})}_{k}\leftarrow \mathrm{RegPI}(\pi^{(\texttt{T})}_{k-1}, \gamma, \tilde P^{(\texttt{T})}, \pi^{(\texttt{S})}, \lambda)$
            \ENDFOR
        }
    \end{algorithmic}
\end{algorithm*}
Here we present the psuedocode (Algorithm~\ref{alg:CFPT}) outlining our proposed Counterfactually Guided Policy Transfer (CFPT) approach as discussed in this section. CFPT is enabled by first having access to an optimal treatment policy $\pi^{(\ttS)}$ developed within a data-rich source environment $\ttS$ as well as an estimation of the transition statistics $P^{(\ttS)}$ collected from observed data. These methods combine to form a two-phase counterfactual regularization approach for policy learning in a data-scarce target environment $\ttT$.

Policy learning is done through a counterfactually regularized form of PI (CF-PI). The heart of CF-PI rests on the discussion provided in Section~\ref{sec:reg_pi} which introduces how we regularize PI (RegPI) in the target environment through KL-divergence log aggregation. CF-PI is executed as follows. For $K$ iterations, a batch of trajectories $\{h^i\}_{i=1}^N$ observed within the target environment are sampled (Alg.~\ref{alg:CFPT}, line 24). This batch is used, along with the current policy within $\ttT$, $\pi_k^{(\ttT)}$, and the prior over the transition statistics from the source environment $P^{(\ttS)}$ to generate counterfactual trajectories $\{\tau^i\}_{i=1}^N$ (Alg.~\ref{alg:CFPT}, line 25 $\rightarrow\mathrm{CFI}$ lines 1-6). This counterfactual sampling procedure, leveraging the property of counterfactual stability within Gumbel-Max SCMs, is described in Sections~\ref{sec:aug_gumbel}. 
The batch of trajectories produced may exhibit some diversity in observed transition statistics from those observed in $\ttT$. To account for this, an augmented transition matrix $\tilde{P}^{(\ttT)}$ is formed through a weighted sum between $P^{(\ttT)}$ and the empirically observed set from $\{\tau^i\}_{i=1}^N$ ($\hat{P}^{(\ttT)}$, line 26). This augmented transition matrix is then passed to RegPI as discussed in Section~\ref{sec:reg_pi} (line 27 $\rightarrow\mathrm{RegPI}$, lines 7-21). 

RegPI alternates between policy evaluation and policy improvement steps. In policy evaluation (lines 11-15) where the current policy $\pi_k^{(\ttT)}$ is used to refine an estimate of the underlying value function based on the observed rewards and estimated transition statistics when applying $\pi^{(\ttT)}_k$. Once this value estimate converges, it is used in a form of a Bellman update (line 17) to generate a proposal distribution over actions for each state. This is the beginning of the policy improvement step (lines 16-19). After the proposal distribution $\nu(\cdot|s)$ is generated, it is used to estimate the best policy while being constrained by the source policy $\pi^{(\ttS)}$ through KL-divergence log-aggregation (line 18). This improved policy is then sent back to the evaluation step to refine the estimate of the value function and this process continues until $\pi_k^{(\ttT)}$ converges or a maximum number of iterations has been performed. With this updated policy, a new batch of trajectories are sampled from $\mathcal{H}^{(\ttT)}$ to draw new counterfactual samples and next iteration continues to further optimize the target policy $\pi^{(T)}$.

\section{Additional Experimental Details and Results}\label{sec:app_experiments}

This section contains information about specific settings used to learn our policies using the various baseline approaches as well as the ablations and full CFPT procedure. We also present additional experimental findings in support of those presented in the main body of the paper.

\subsection{Baseline Policy Learning Settings}

As mentioned, we use the coarse sepsis simulator introduced by \citet{oberst2019counterfactual} which can be found at \url{https://www.github.com/clinicalml/gumbel-max-scm}. We make one major deviation from their setting of the simulator in that we do not mask out the observations of a patient's glucose level. We also adjust the initialized proportion of diabetic patients included in the population used to define an experimental environment.

For all experiments and baselines, we fix the discount rate $\gamma$ to $0.99$ and the maximum number of iterations for each use of policy iteration to $1000$. The number of trajectories in the source environment $\ttS$ was fixed to $10,000$ and the proportion of diabetic patients in $\ttS$ was set to $0.1$. All target environments $\ttT$, independent of the size of the diabetic subpopulation, were represented with $2000$ trajectories. Recall that any indication of whether a patient has diabetes or not is unobserved.

In the following subsections, we report any additional parameter settings or adjustments to the learning procedure. All policy learning is done via Policy Iteration (augmented as described in the paper) utilizing an adjusted version of the \texttt{pymdptoolbox} library. Code to replicate our experiments will be made available upon publication of our paper.

\subsubsection{Baselines}
\paragraph{\sc Random}
This baseline doesn't explicitly learn a policy. For evaluation, all action selection is done by uniformly sampling between the 8 possible actions. 

\paragraph{\sc Scratch}
This non-transfer baseline constructs an empirical transition matrix from the observed data $\mathcal{H}^{(\ttT)}$ which is then used within policy iteration to produce the policy $\pi^{(\ttT)}$.

\paragraph{\sc Pooled}
To pool the data between the environments $\ttS$ and $\ttT$ we estimate the transition statistics using both $\mathcal{H}^{(\ttS)}$ and $\mathcal{H}^{(\ttT)}$ which is then used to learn a policy with Policy Iteration in the target environment $\ttT$.

\paragraph{\sc Blind}
This naive transfer baseline does not learn a new policy, rather it blindly uses the policy $\pi^{(\ttS)}$ from the source environment without any adaptation or fine tuning. In evaluation within the target environment, actions are selected according to the distribution put forward by the source policy.

\subsubsection{Counterfactually Guided Policy Transfer (CFPT)}
\paragraph{\sc CFPT}

When applying CFPT for learning a policy in the target environment we needed to tune several hyperparameters to set-up the best policy learning environment within a data-scarce target environment $\ttT$ when transferring from a fixed source environment (proportion of diabetic patients: 0.1). This involved determining the best value for the number of iterations $K$ of CF-PI, the mixture weight for regularizing the counterfactual sampling $w^{(\ttT)}$, the weighting for augmenting the observed transition statistics $\eta$, and perhaps most importantly the weight for regularizing the policy learning with $\lambda$. As $w^{(T)},\ \eta\text{ and }\lambda$ correspond to linear combinations between two quantities, we tested each of these hyperparameters between 0 and 1 in increments of 0.1, using the learned policy's true RL performance in the target environment to compare between settings. We report the optimal settings for learning within $\ttT$ in each target environment (diabetic proportion of population ranging from 0 to 1 in 0.1 increments) in Table~\ref{tab:app_parameters}. For all target environments the number of iterations $K$ for CF-PI was $50$.

\begin{table*}
  \caption{Best performing hyperparameter settings for CFPT across each target environment $\ttT$}
  \label{tab:app_parameters}
  \centering
  \begin{tabular}{c|c|c|c|c|c|c|c|c|c|c|c|c|c}
    \toprule
    Diabetic Proportion & 0.0 & 0.1 & 0.2 & 0.3 & 0.4 & 0.5 & 0.6 & 0.7 & 0.8 & 0.9 & 1.0 \\    
    \midrule
    $w^{(\ttT)}$ & 0.8 & 0.8 & 0.8 & 0.6 & 0.7 & 0.8 & 0.8 & 0.6 & 0.8 & 0.7 & 0.8 \\ 
    $\eta$ & 0.7 & 0.8 & 0.7 & 0.7 & 0.8 & 0.7 & 0.6 & 0.8 & 0.7 & 0.7 & 0.7 \\ $\lambda$ & 0.9 & 0.9 & 0.3 & 0.1 & 0.3 & 0.6 & 0.3 & 0.1 & 0.3 & 0.4 & 0.9 \\
    \bottomrule
  \end{tabular}
\end{table*}

\subsubsection{Ablations}
\paragraph{\sc Reduced CFPT}
In this ablation of CFPT, we removed the informative prior over the transition statistics within counterfactual sampling. This effectively removes this form of regularization that makes up CFPT. All other procedures and operations within CFPT were run as normal with the same parameter settings as shown in Table~\ref{tab:app_parameters} performing best.

\paragraph{\sc Regularized Policy Iteration (RegPI)}

In this ablation, we removed the sampling of counterfactual trajectories completely from CFPT. We also removed any batch sampling from $\mathcal{H}^{(T)}$, using instead the full set of observed data within $\ttT$. A single run of RegPI was executed, using the top performing values for $\lambda$ as reported in Table~\ref{tab:app_parameters}.

\subsection{Additional Results}
\label{sec:app_addtl_results}

In this section we present additional results that we did not have space to include in the main paper as well as an important additional analysis over the separate subpopulations (diabetic vs. non-diabetic) among the patients observed in the target clinical environment. In Table~\ref{tab:app_fig2} we present the numerical values for the comparison between CFPT and all baselines and ablations shown in Figure~\ref{fig:comp_rl_results}.

\begin{table}[ht!]
\centering
\begin{tabular}{ccc}
\toprule
\bf{Approach}   & \bf{True RL Reward}      & \bf{WIS Reward}        \\ \hline
Scratch  & $-0.7398 \pm 0.007$ & $0.6388 \pm 0.584$  \\
Pooled   & $-0.4808 \pm 0.012$ & $0.9782 \pm 0.004$  \\
Blind    & $-0.3915 \pm 0.013$ & $0.5874 \pm 0.113$ \\ \hline
RegPI    & $-0.3366 \pm 0.012$ & $0.6266 \pm 0.057$  \\
Red. CFPT & $-0.2116 \pm 0.010$ & $0.7689 \pm 0.077$  \\
CFPT      & $-0.1491 \pm 0.011$ & $0.7333 \pm 0.004$ \\ 
\bottomrule
\end{tabular}
\vspace{0.15cm}
\caption{\small Numerical values corresponding the policy performance results presented in Figure~\ref{fig:comp_rl_results}. The observed behavior policy $\mu^{(\ttT)}$ receives an average reward of $0.1486 \pm 0.018$.}
\label{tab:app_fig2}
\end{table}

\subsubsection{Analysis of selecting $\eta$, affecting the augmentation of $P^{(\ttT)}$}
\label{sec:app_eta_analysis} In Figure~\ref{fig:app_eta_sweep} we demonstrate the range of policy performance under CFPT with CF-PI when varying the parameter $\eta$. Recall from Section~\ref{sec:full_method} that $\eta$ is used to weight the augmentation of the observed transition statistics in the target environment ($P^{(\ttT)}$) with those estimated from the counterfactually inferred trajectories ($\hat{P}^{(\ttT)}$). In this figure we demonstrate CFPT performance for policies learned in the simulated environment with a proportion of diabetic patients being 0.8, transferring from a source environment where the diabetic proportion is 0.1. The number of iterations $K$ of CF-PI is set to 50 and we demonstrate the effect of the policy regularization parameter $\lambda$ and the parameter $\eta$ which is used to incorporate the inferred empirical transition matrix $\hat{P}^{(\ttT)}$ into the observed target transition matrix $P^{(\ttT)}$ for use in regularized policy iteration (Algorithm~\ref{alg:CFPT} line 26).

\begin{figure}[htbp!]
    \centering
    \includegraphics[width=\linewidth]{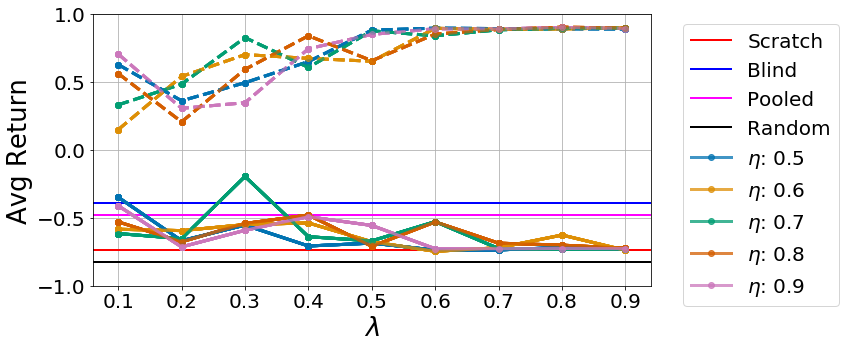}
    \caption{Demonstration of parametric study used to identify optimal settings of CFPT parameters. Shown here, within a target environment with a diabetic proportion set to 0.8 with a source population diabetic proportion set to 0.1, we see that the True RL performance (solid lines) varies as $\lambda$ and $\eta$ interact with a diminshed effect as $\lambda$ increases. CF-PE estimated reward (dotted lines) asymptotically overestimates policy performance as $\lambda$ increases.}
    \label{fig:app_eta_sweep}
\end{figure}

What we see in Figure~\ref{fig:app_eta_sweep} is that there is a balance when selecting $\eta$ and $\lambda$ for CFPT policy learning. As $\lambda$ increases, meaning we are using less of the source environment, no matter the choice of $\eta$, performance more or less converges to the baseline non-transfer setting within $\ttT$. However when $\lambda$ is smaller, meaning we intend to use a larger proportion of the source policy, we see that the choice of $\eta$ can have a broad effect. In the scenario demonstrated in Figure~\ref{fig:app_eta_sweep}, we see that the optimal setting comes when $\eta=0.7$ and $\lambda=0.3$ which are the values used for all CFPT variants and ablations presented in Sec~\ref{sec:experiments} when the proportion of diabetic patients in $\ttT$ is 0.8.

\subsubsection{Counterfactual sampling with fully observed state}
\label{sec:app_counterfactual_samples}

Similar to the analysis presented in Section~\ref{sec:counterfactual_samples} and in Figure~\ref{fig:maskedDiab}, we investigate the change in the feature distributions in $\ttT$ when the simulated patient's diabetic status is known after sampling counterfactual trajectories using the Gumbel-Max SCM, regularized by $\ttS$. The resulting comparison is shown in Figure~\ref{fig:fullObs}.

\begin{figure}[h!]
    \centering
    \includegraphics[width=0.9\linewidth]{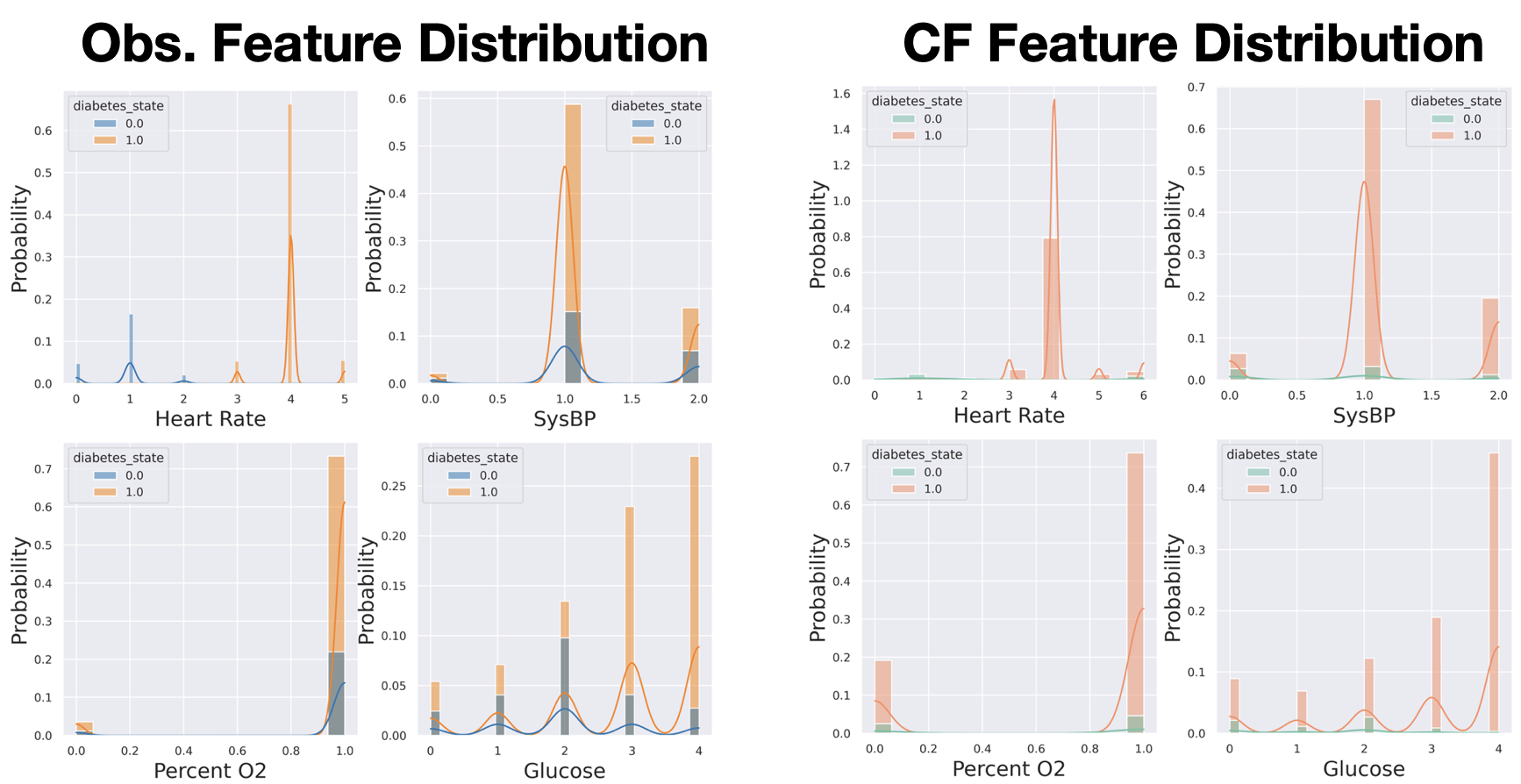}
    \caption{Feature distributions with full observations with the patient observations obtained in $\ttT$ on the left and a resampling of the feature distributions using counterfactuals drawn from the regularized Gumbel-Max SCM on the right. }
    \label{fig:fullObs}
\end{figure}

\subsubsection{Off-policy Evaluation of $\pi^{(\ttT)}$}

The off-policy evaluation (OPE) of policies learned from fixed data, without the ability to independently test them is a challenging part of offline RL, and has been understudied in partially observed settings~{\small\citep{tennenholtz2020off}}. Importance Sampling (IS) can provide an estimate of policy performance with low bias for OPE~{\small\citep{thomas2015safe}}, which is desirable in a transfer setting. While we focus on true rewards as our primary evaluation in this paper, we provide OPE estimates in this section for completeness. For this, we use weighted importance sampling (\ope)~{\small\citep{mahmood2014weighted}} to evaluate our transfer policies due to its interesting consistency properties. In Sec.~\ref{sec:app_policy_introspection} we use an alternative, counterfacutally determined OPE method, CF-PE~{\small\citep{oberst2019counterfactual}}, to qualitatively evaluate the learned policies. 

OPE estimates generally exhibit significant overconfidence in expected rewards, as areas of high reward are erroneously extrapolated over unseen regions of the state space. We report the results of evaluating \ope~for the learned policies $\pi^{(\ttT)}$ in Table~\ref{tab:app_fig2} for the setting of $\ttT$ with a 0.8 proportion of diabetic patient trajectories. As expected, WIS overestimates the true RL return in $\ttT$, even with poor policies (i.e. {\bf{\small{\textsc{Scratch}}}}). However, we see some semblance of improvement with each component used to implement our proposed CFPT approach. However, the general unreliability of these OPE estimates make it difficult to truly evaluate the benefits of transfer with counterfactual regularization.

Fortunately, CF-PE allows for the comparison of individual counterfactual trajectories influenced by CFPT and other methods. This form of introspective evaluation can help identify glaring safety issues for deployment of a trained policy in a new environment. As seen in Sec.~\ref{sec:app_policy_introspection}, CFPT acts more conservatively and closely approximates the observed behavior, leading to more stable performance.

\begin{figure}[htbp!]
    \centering
    \includegraphics[width=0.5\textwidth]{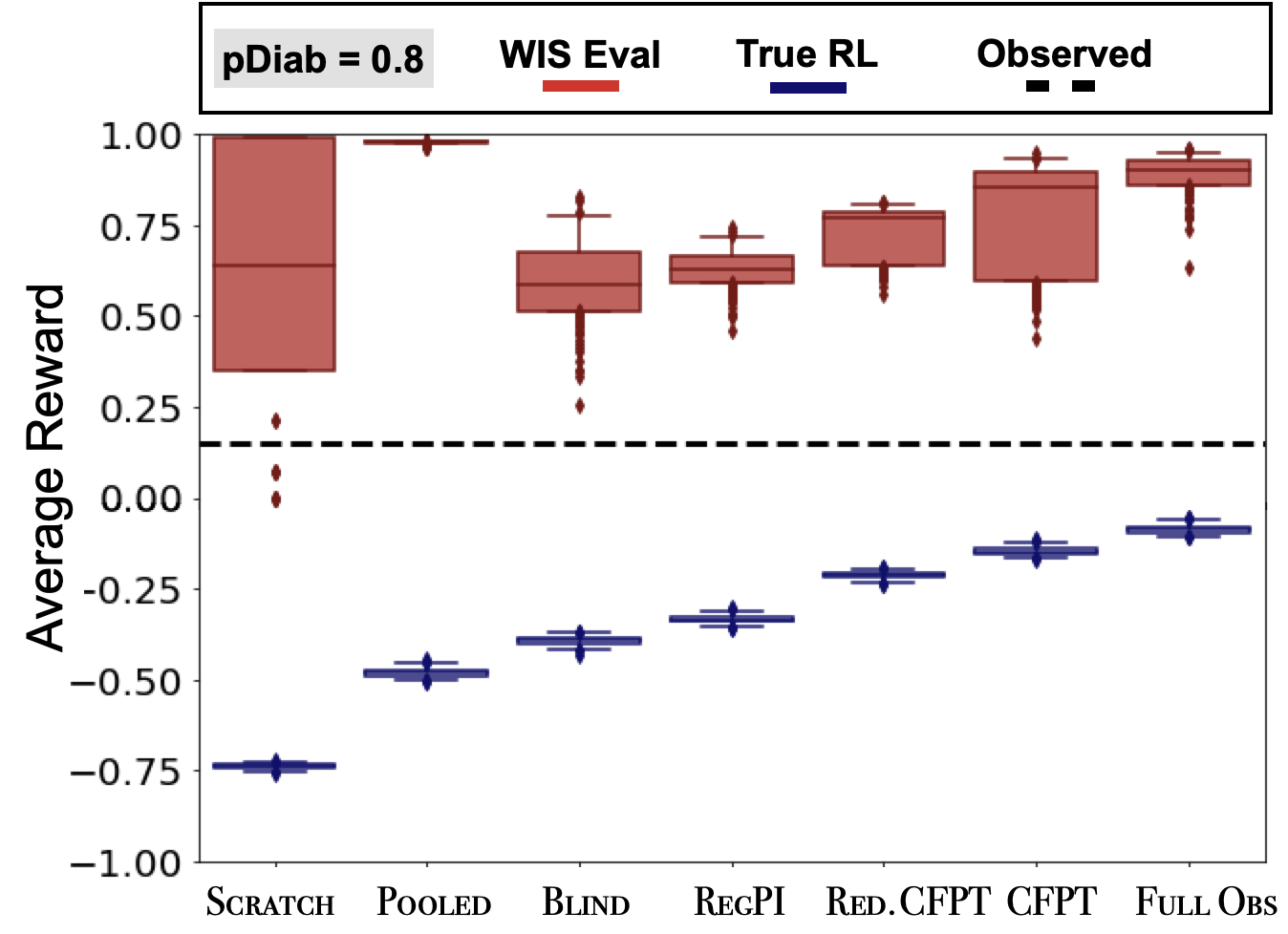}
    \caption{\small Comparison of estimated reward in $\ttT$ between CFPT and the baselines outlined in Sec.~\ref{sec:baselines}. Results after {\small \ope} are plotted in red where the true performance in $\ttT$ is plotted in blue. 95\% uncertainty intervals are found through 100 bootstrapped samples of the 5000 generated trajectories under the learned target policy.}
    \label{fig:app_rl_perf_ope}
\end{figure}

\subsection{Qualitative Analysis of $\pi^{(\ttT)}$}
\label{sec:app_qualitative_analysis}

\paragraph{Treatment Selection under CFPT:} To better compare policy evaluations between baselines, we perform an introspective analysis using CF-PE on both a policy and trajectory level. First, we compare the counterfactual outcomes between the naive baseline policy without transfer (\textsc{Scratch}) against our full CFPT trained policy, to identify how CFPT improves policy learning within $\ttT$ (other comparisons between CFPT and the baselines are in Section~\ref{sec:app_policy_introspection}). We first compare the counterfactual outcomes as estimated through CF-PE and then compare policy behavior under counterfactual evaluation for an individual patient drawn from $\ttT$. In Section~\ref{sec:app_cfpe} we present the aggregate counterfactual outcomes as suggested by CF-PE in comparison to what was observed. The primary difference in the evaluation between the \textsc{Scratch} policy and that learned through CFPT is in the percentage of patients CFPT does not discharge while {\sc Scratch} does. To further identify what separates these two policies we select patients who die under the behavior policy but are inferred to be discharged under {\sc Scratch} but kept in the hospital under CFPT. In Figure~\ref{fig:comp_cf_outcomes}, we observe that the non-transfer baseline ({\sc Scratch}) is far more aggressive in it's treatment decisions, leading to premature treatment cessation as the patient's condition deteriorates (visualized by the blue counterfactual trajectories) immediately after they are indicated for discharge. In contrast, the CFPT policy chooses a strategy that stably maintains the patient condition, continuing all treatments until the observation window terminates. 

\begin{figure*}[!ht]
    \centering
    \includegraphics[width=0.75\textwidth]{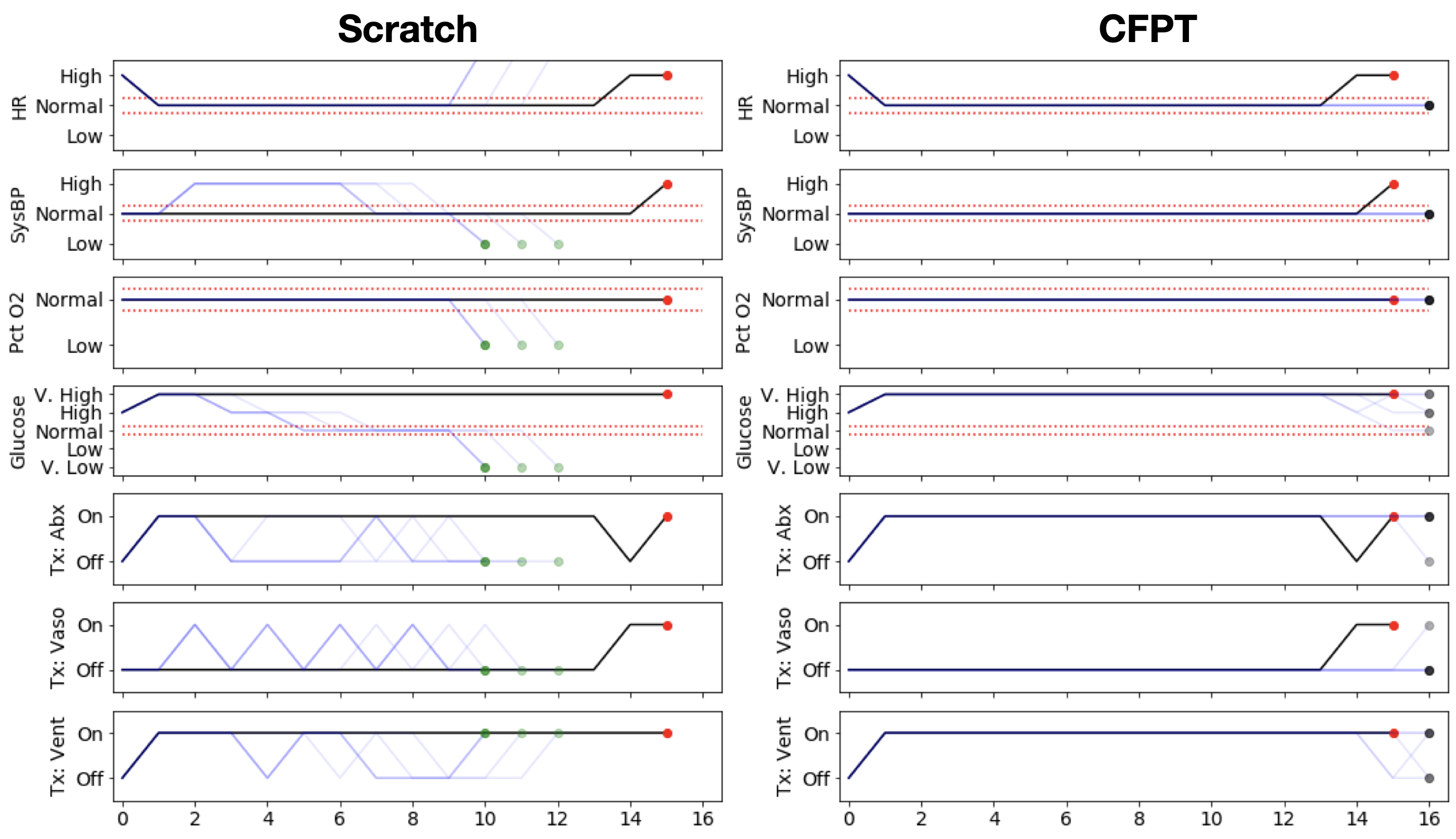}
    \caption{Qualitative comparison between CFPT and the {\small \sc Scratch} baseline. We compare an individual patient's counterfactual trajectories using these policies. Dark lines are the observed vital measurements and actions over time while the lighter blue traces correspond to counterfactual observations and actions. Green, red and black markers denote discharge, death and no change respectively. CFPT provides more stable treatment selection in comparison with the non-transfer baseline. Additional samples in Appendix~\ref{sec:app_policy_introspection}}.
    \vspace{-0.35cm}
    \label{fig:comp_cf_outcomes}
\end{figure*}

\subsubsection{Sub-population Analysis of Evaluated Policies}
\label{sec:app_subpopulation_analysis}

In Figure~\ref{fig:app_diab_subpop} we demonstrate the differences among subpopulations when learning a policy with CFPT for different target environments $\ttT$ (we choose to present here the subpopulations from environments with a proportion of diabetic (pDiab) patients being 0.3, 0.5 and 0.8). When pDiab = 0.5, the performance of CFPT is only marginally better than the compared baselines. It's evaluated policy performance with CF-PE is also on par with the non-transfer baseline ({\sc Scratch}) which is also mirrored in the aggregate counterfactual outcomes shown here as it is comparable to what has been observed when evaluating the {\sc Scratch} baseline previously. The comparison between the two highest performing instances of CFPT (pDiab = 0.2 and pDiab = 0.8) is an interesting cross-section view of what happens when the target environment differs from the source environment. Recall that the source environment for all instances of transfer was set to pDiab = 0.1. The population of this source environment is distributionally similar to $\ttT$ when pDiab=0.2. Here, we see a significant increase in the number of patients who are neither discharged or die in counterfactual evaluation, in comparison to the other two pDiab settings in Figure~\ref{fig:app_diab_subpop}. This provides some further evidence toward our conclusion that CFPT aids in the development of more circumspect policies. 

\begin{figure*}[htbp!]
    \centering
    \includegraphics[width=\textwidth]{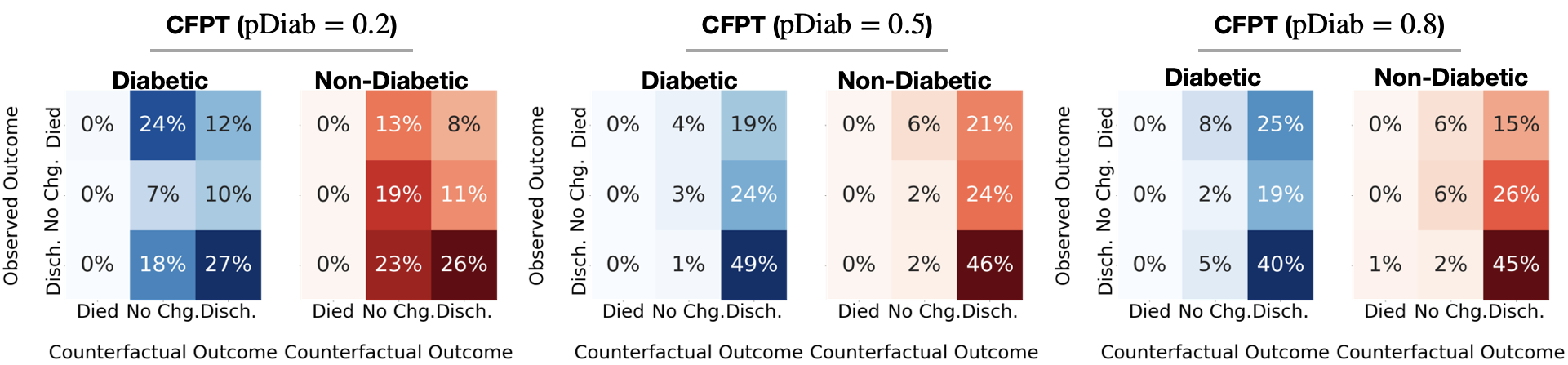}
    \caption{Aggregated counterfactual outcomes by subpopulation for different target environments $\ttT$.}
    \label{fig:app_diab_subpop}
\end{figure*}

In Figure~\ref{fig:app_eta_subpop} we demonstrate the differences among subpopulations when learning a policy with CFPT having different settings of $\eta$ (see Section~\ref{sec:app_eta_analysis}). With a properly chosen $\eta$ (here, 0.7), we see that the evaluated outcomes of the policy increasingly push toward discharge while less optimal policies (as evaluated) appear to not have identified appropriate treatment strategies to move a majority of the observed patient trajectories toward discharge. This is most apparent when considering the non-diabetic patients, those who are in the minority within the target environment. This divergence in performance between subpopulations speaks to the importance of properly tuning the CFPT procedure.

\begin{figure*}[htbp!]
    \centering
    \includegraphics[width=\textwidth]{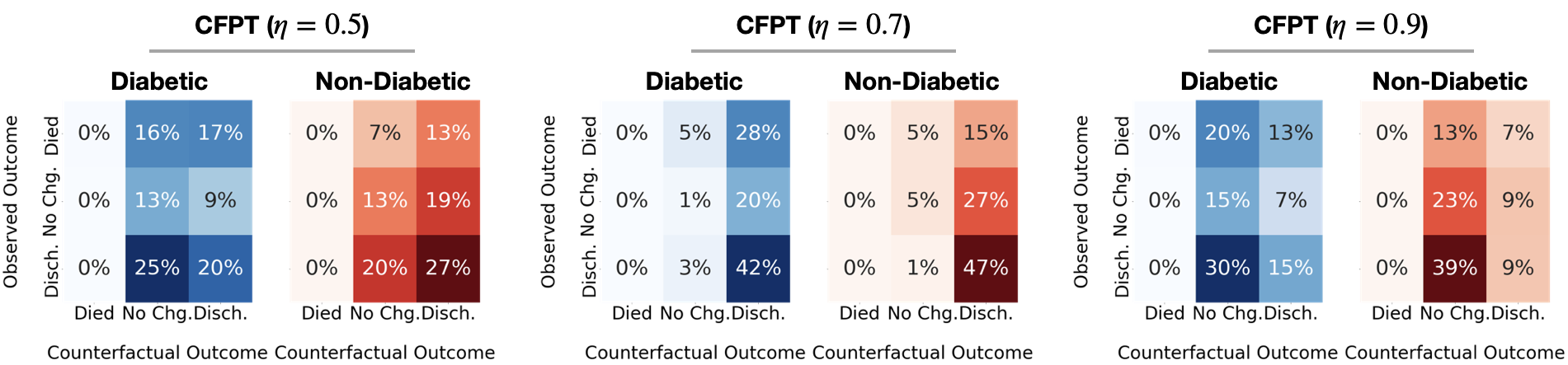}
    \caption{Aggregated counterfactual outcomes by subpopulation for different settings of $\eta$ within CFPT.}
    \label{fig:app_eta_subpop}
\end{figure*}

Figure~\ref{fig:app_subpop_eval} presents an analysis between subpopulations for the non-transfer baseline ({\sc Scratch}) and our proposed CFPT approach. Here we're looking at outcomes as inferred by counterfactual policy evaluation for the policies learned for each approach. As was discussed in Section~\ref{sec:results}, the policies learned via CFPT are slightly more conservative for the rarely observed non-diabetic population of the target environment. The suggested treatments and the inferred outcomes are far more measured in aggregate when using CFPT than is manifest from the non-transfer baseline.

\begin{figure}[htbp!]
    \centering
    \includegraphics[width=0.7\linewidth]{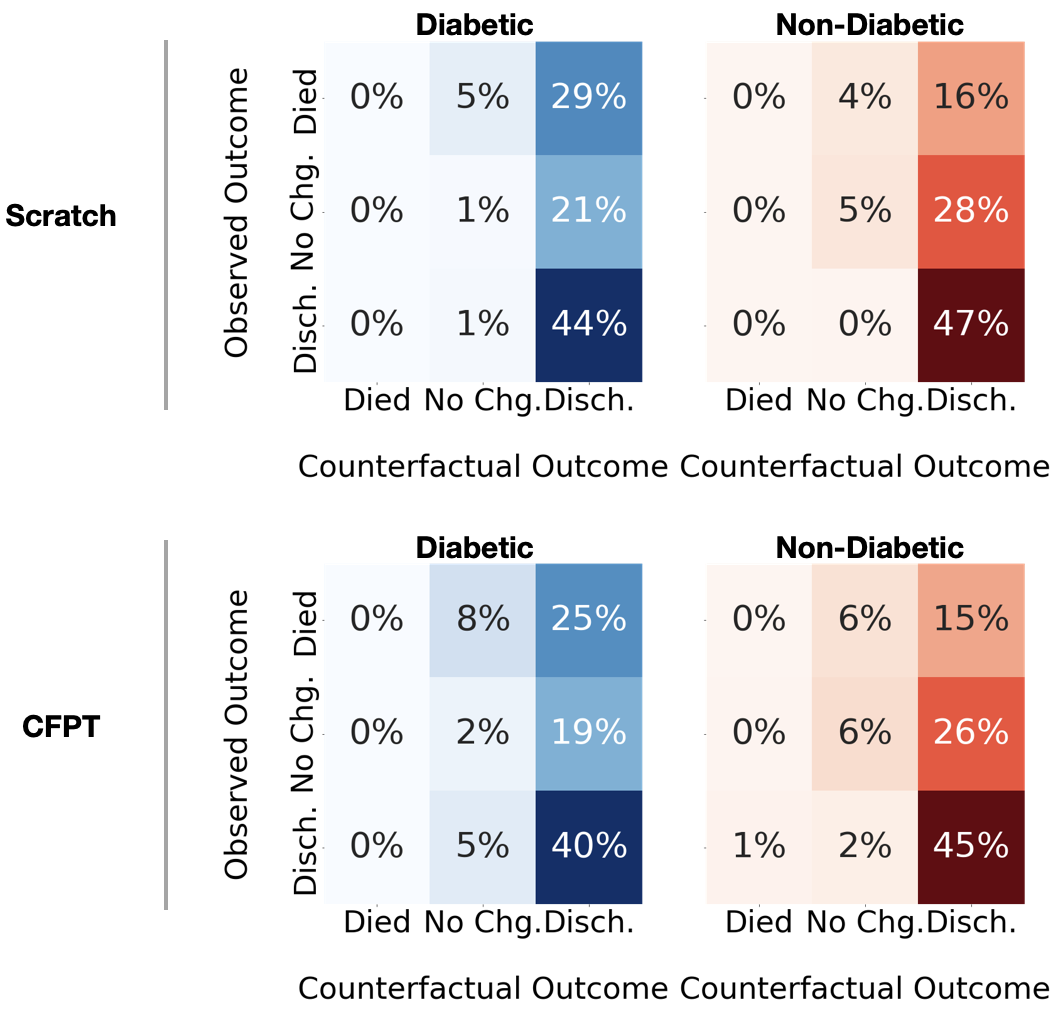}
    \caption{Aggregated counterfactual outcomes by subpopulation following the non-transfer baseline policy vs CFPT. These values are normalized by the number of patients belonging to each subpopulation (diabetic vs. non-diabetic) respectively. CFPT in aggregate is more conservative for the diabetic (rare class in source) in CF-PE evaluation.}
    \label{fig:app_subpop_eval}
\end{figure}

\subsubsection{Counterfactual policy evaluation: full comparison}
\label{sec:app_cfpe}

In Figure~\ref{fig:app_cfpe} we present a full comparison between the counterfactual policy evaluation results, segmented by outcome, for each baseline and version of our proposed CFPT approach for off-policy transfer learning with limited data in the target environment. The counterfactual outcome demonstrates the unreliability of a blind transfer policy. Benefits of each parts of our regularization do shift the confidence in our policy toward discharge. 

\begin{figure*}[htbp!]
    \centering
    \includegraphics[width=0.75\textwidth]{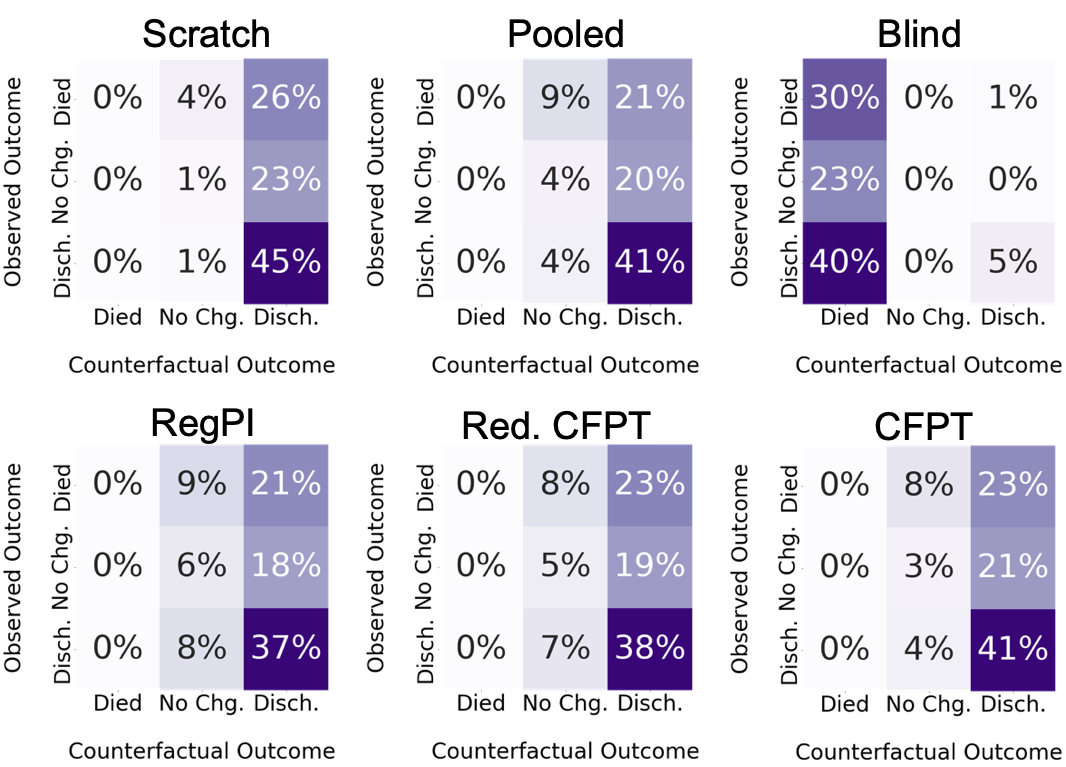}
    \caption{Comparison of all baselines in their aggregate population statistics in counterfactual evaluation of the policies learned in the target environment pDiab=0.8}
    \label{fig:app_cfpe}
\end{figure*}

\subsubsection{Introspective Analyses of Learned Policies}
\label{sec:app_policy_introspection}
In this section we include additional introspective trajectory comparisons between the the non-transfer baseline ({\sc{Scratch}}) and our proposed transfer procedure ({\sc{CFPT}}). The simulated patients extracted for this comparison are those that were observed to die where the {\sc Scratch} baseline is evaluated to have treated these patients sufficiently to be discharged while CFPT is more circumspect, being evaluated to have sustained the patient's life yet not able to move them to be discharged. These examples confirm the insight reported in the main text of the paper, that the policy learned through CFPT more closely approximates the observed behavior policy in a stable fashion while also seeing slight deviations that appear to contribute to keeping the patient's vitals within a healthy range. In comparison, the non-transfer baseline policy proposes far more aggressive treatments that, in off-policy evaluation, appear to be effective yes the patient's vitals rapidly fall out of a normal or healthy range as soon as all treatments are stopped.

\begin{figure*}
    \centering
    \includegraphics[width=0.75\textwidth]{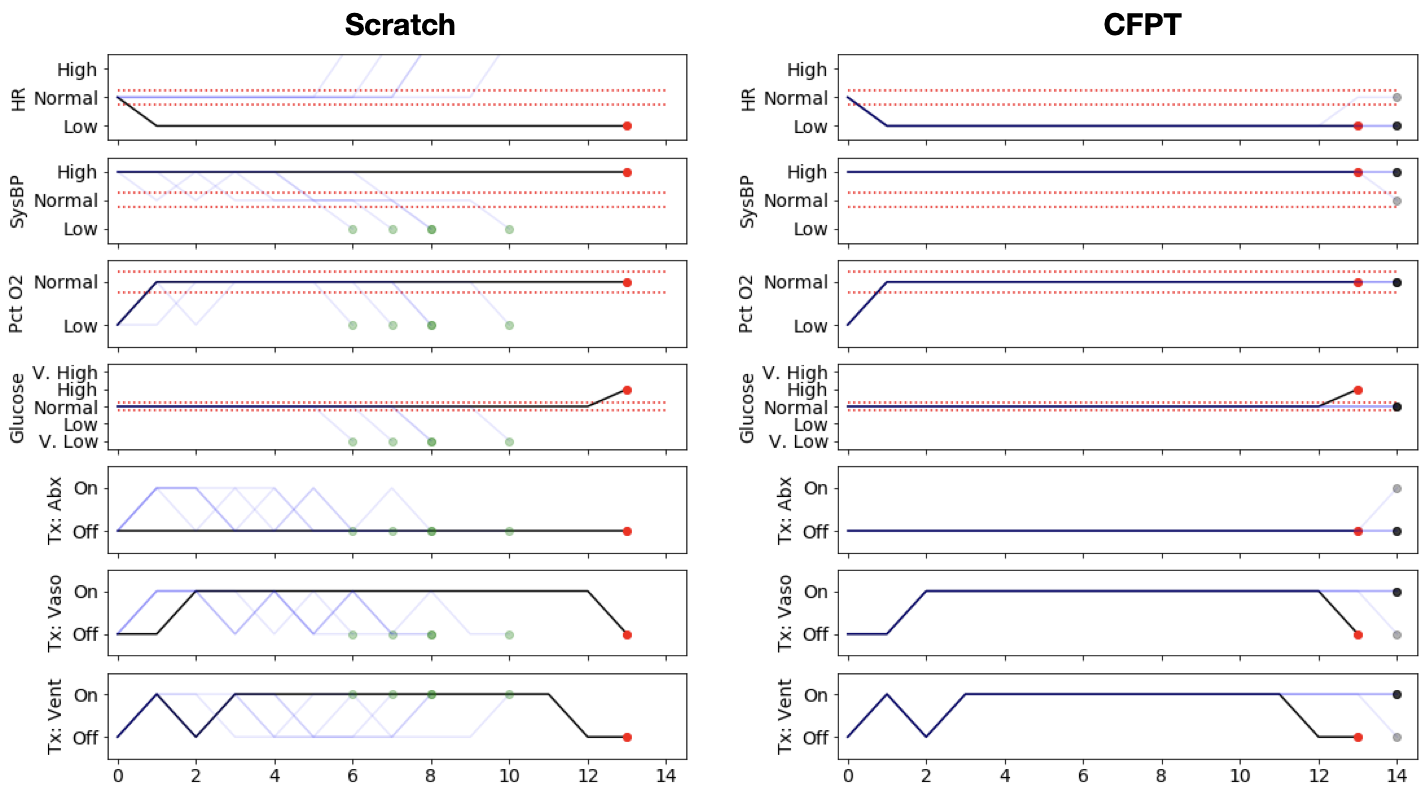}
    \caption{Introspective analysis of counterfactually sampled trajectories following the non-transfer baseline policy evaluation (left) compared with the evaluation of the proposed CFPT policy (right). This simulated patient is non-diabetic.}
    \label{fig:app_nondiab_pt}
\end{figure*}

To augment the presentation provided in Figure~\ref{fig:comp_cf_outcomes} we include four additional trajectory introspection figures. The first of which belongs to a non-diabetic patient (Figure~\ref{fig:app_nondiab_pt} recall, this is type of patient is found in lower proportion within the target environment) while the other three are diabetic patients (Figures~\ref{fig:app_diab_pt1}-~\ref{fig:app_diab_pt3}).

\begin{figure*}
    \centering
    \includegraphics[width=0.75\textwidth]{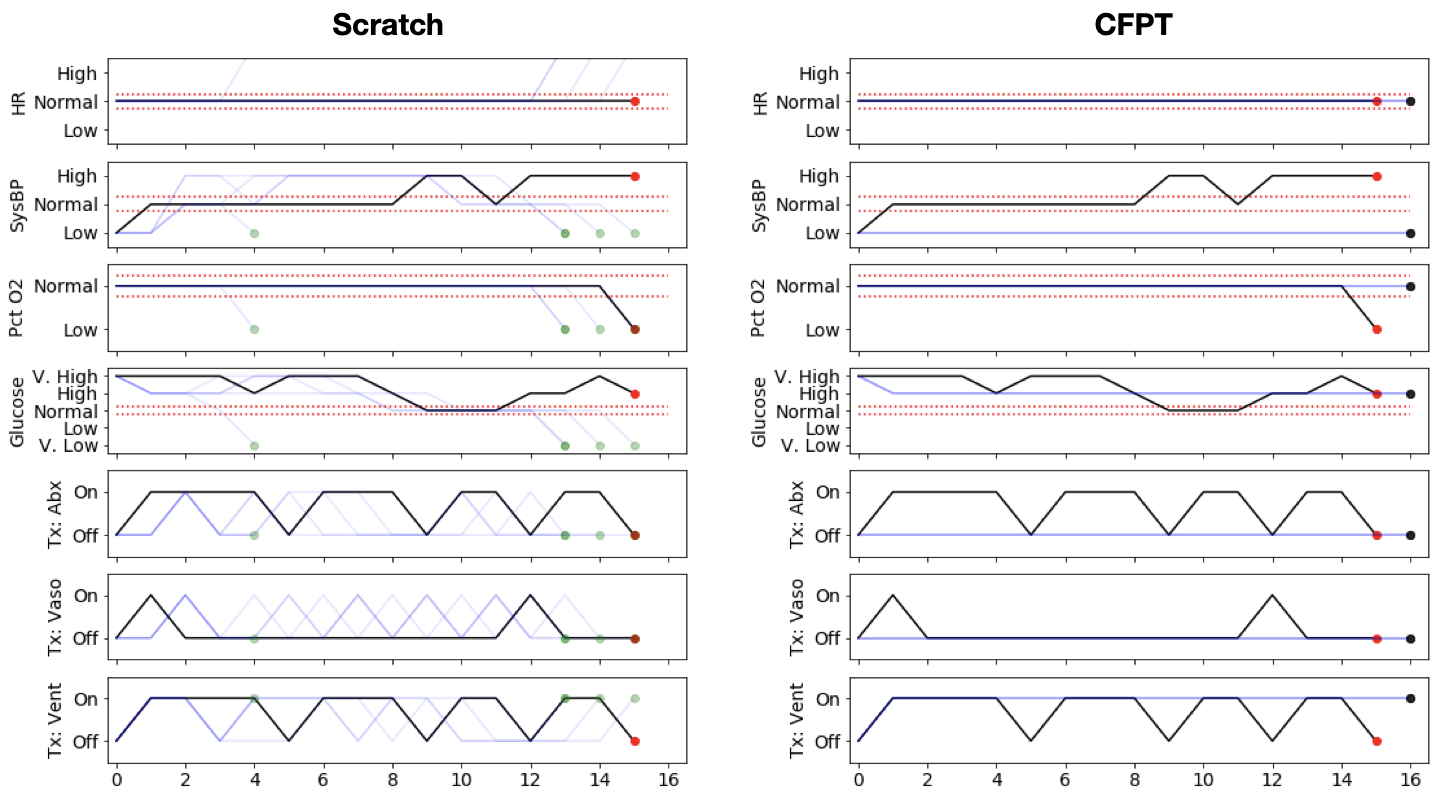}
    \caption{Introspective analysis of counterfactually sampled trajectories following the non-transfer baseline policy evaluation (left) compared with the evaluation of the proposed CFPT policy (right). This simulated patient is diabetic.}
    \label{fig:app_diab_pt1}
\end{figure*}

\begin{figure*}
    \centering
    \includegraphics[width=0.75\textwidth]{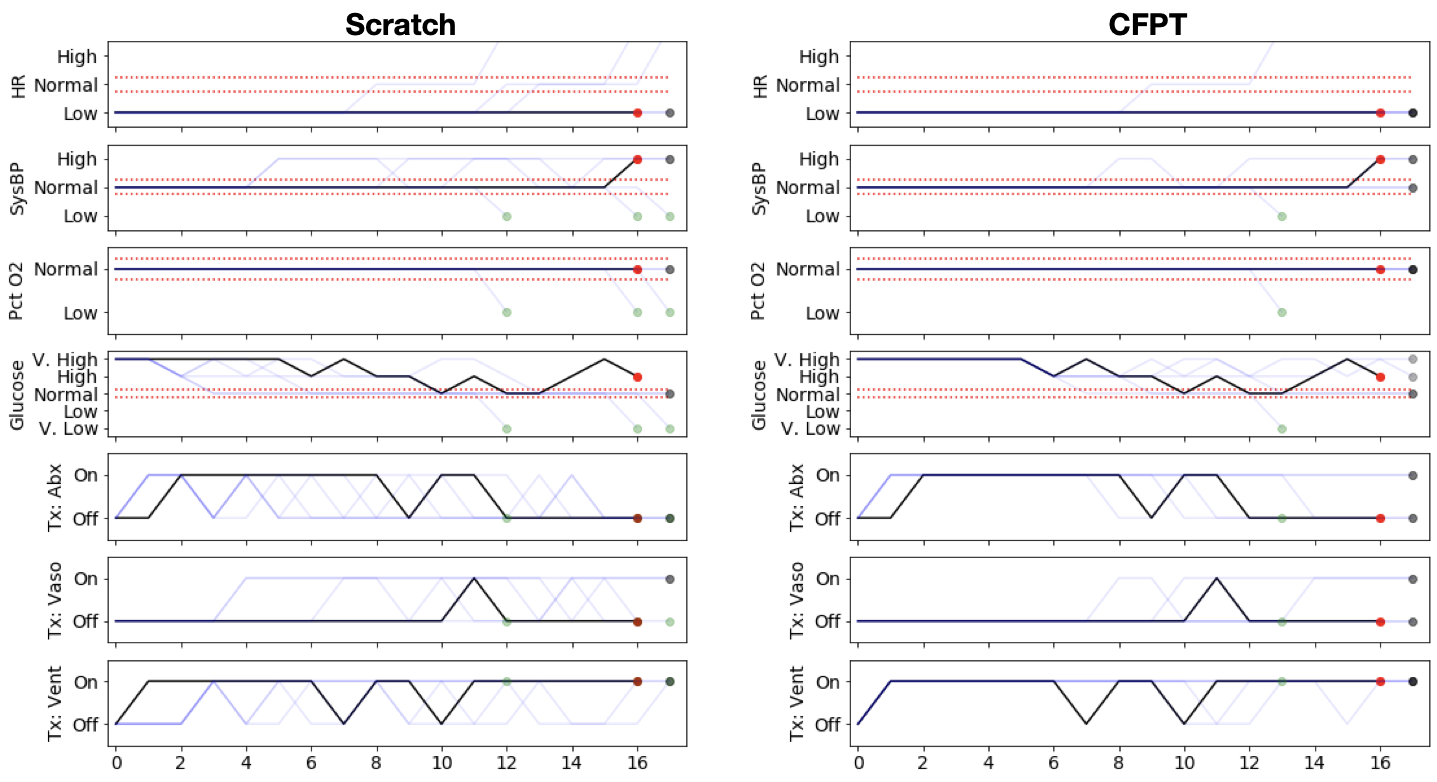}
    \caption{Introspective analysis of counterfactually sampled trajectories following the non-transfer baseline policy evaluation (left) compared with the evaluation of the proposed CFPT policy (right). This simulated patient is diabetic.}
    \label{fig:app_diab_pt2}
\end{figure*}

\begin{figure*}
    \centering
    \includegraphics[width=0.75\textwidth]{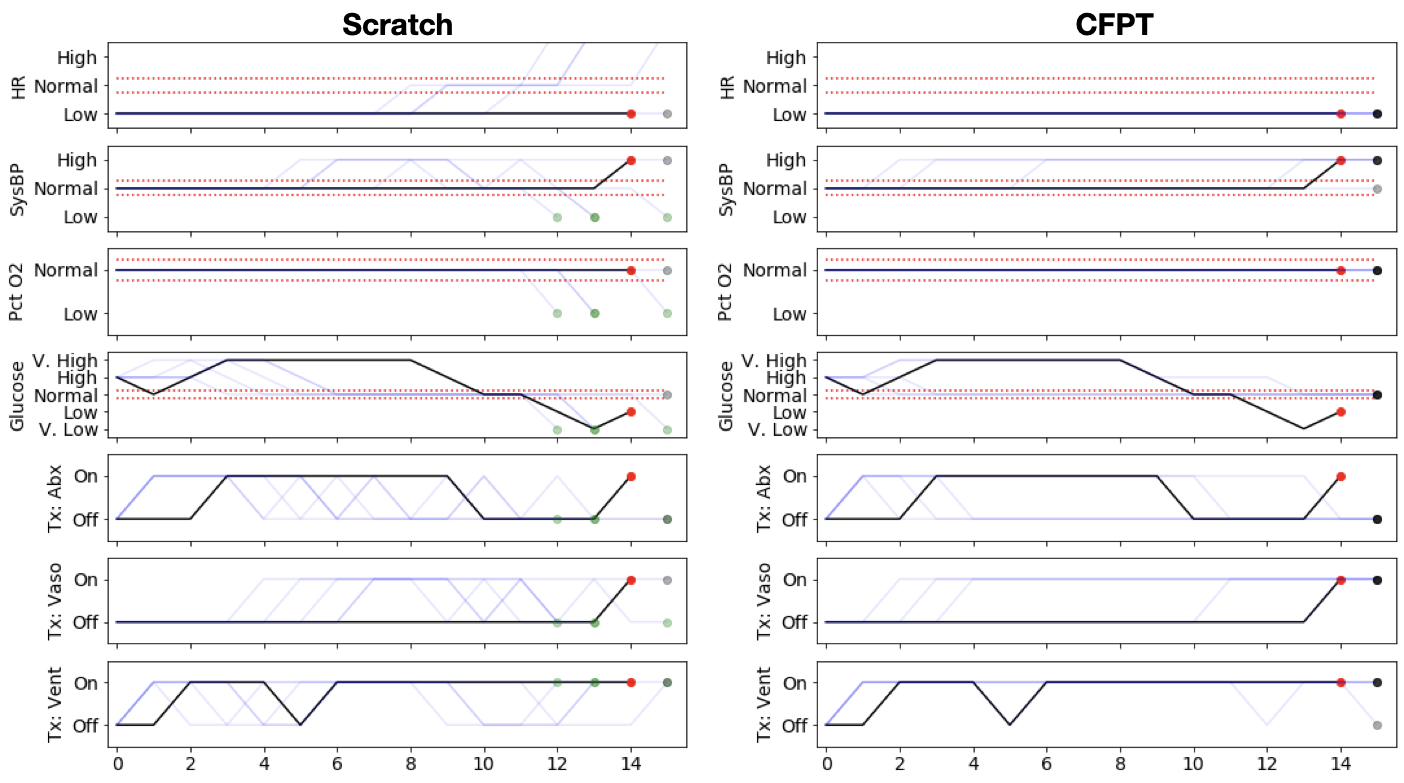}
    \caption{Introspective analysis of counterfactually sampled trajectories following the non-transfer baseline policy evaluation (left) compared with the evaluation of the proposed CFPT policy (right). This simulated patient is diabetic.}
    \label{fig:app_diab_pt3}
\end{figure*}

\end{document}